\documentclass[pmlr]{jmlr}
\usepackage{caption} 

\RequirePackage{graphicx}
 \usepackage{booktabs}
\usepackage{longtable}
 \usepackage{float}
\usepackage{subcaption}
\usepackage{algorithm}
\usepackage{algorithmic}
\usepackage{graphicx} 
\usepackage{subcaption} 
\usepackage{amsmath}
\usepackage{tcolorbox}
\usepackage{amsfonts}	
\usepackage{amssymb} 
\usepackage{mathtools}	
\newtheorem{prop}{Proposition}
\usepackage{nicefrac}       
\usepackage{natbib}
\usepackage{tikz}
\usetikzlibrary{fit, positioning}
\usepackage{adjustbox}
\usepackage{multirow} 
\usepackage[table]{xcolor}
\allowdisplaybreaks

\newcommand{\ie}{i.e., }

\newcommand{\eg}{e.g., }

\newcommand{\cp}[2]{p{\left( #1 \mid #2 \right)}}

\newcommand{\Ex}[2]{\mathbb{E}_{#1}\left\{ #2 \right\}}

\newcommand{\Cat}[1]{{\rm Cat}\left( #1\right)}

\newcommand{\tr}[1]{\mathrm{tr}\left\{ #1 \right\}} 

\newcommand{\Tb}{\mathbf{T}}

\newcommand{\xb}{\mathbf{x}}

\newcommand{\Xb}{\mathbf{X}}

\newcommand{\argmax}{\mathop{\mathrm{argmax}}}

\newcommand{\expp}[1]{\exp\left\{ #1 \right\}}

\newcommand*\diff{\mathop{}\!\mathrm{d}}

\newcommand{\red}[1]{\textcolor{red}{#1}}

\newcommand{\blue}[1]{\textcolor{blue}{#1}}

\usepackage[dvipsnames,table,xcdraw]{xcolor}

\makeatletter
\def\set@curr@file#1{\def\@curr@file{#1}} 
\makeatother
\usepackage[load-configurations=version-1]{siunitx} 

\theorembodyfont{\upshape}
\theoremheaderfont{\scshape}
\theorempostheader{:}
\theoremsep{\newline}

\jmlrvolume{298}
\jmlryear{2025}
\jmlrworkshop{Machine Learning for Healthcare}

\title[Additive Deep Hazard Analysis Mixtures]{ADHAM: Additive Deep Hazard Analysis Mixtures for Interpretable Survival Regression}

\author{\Name{Mert Ketenci}
\Email{mk4130@columbia.edu}\\ 
\addr Department of Computer Science\\
Columbia University\\
New York, NY, USA 
\AND
\Name{Vincent Jeanselme}
\Email{vj2292@cumc.columbia.edu}\\ 
\addr Department of Biomedical Informatics\\
Columbia University\\
New York, NY, USA 
\AND
\Name{Harry Reyes Nieva}
\Email{hr2479@cumc.columbia.edu}\\ 
\addr Department of Biomedical Informatics\\
Columbia University\\
New York, NY, USA 
\AND
\Name{Shalmali Joshi}
\Email{sj3261@cumc.columbia.edu}\\ 
\addr Department of Biomedical Informatics\\
Columbia University\\
New York, NY, USA 
\AND
\Name{Noémie Elhadad}
\Email{noemie.elhadad@columbia.edu }\\ 
\addr Department of Biomedical Informatics\\
Columbia University\\
New York, NY, USA } 

\begin{document}

\maketitle

\begin{abstract}
Survival analysis is a fundamental tool for modeling time-to-event outcomes in healthcare. Recent advances have introduced flexible neural network approaches for improved predictive performance. However, most of these models do not provide interpretable insights into the association between exposures and the modeled outcomes, a critical requirement for decision-making in clinical practice. To address this limitation, we propose Additive Deep Hazard Analysis Mixtures (ADHAM)\footnote{Code available at \url{github.com/ketencimert/adham}.}, an interpretable additive survival model. ADHAM assumes a conditional latent structure that defines subgroups, each characterized by a combination of covariate-specific hazard functions. To select the number of subgroups, we introduce a post-training refinement that reduces the number of equivalent latent subgroups by merging similar groups. We perform comprehensive studies to demonstrate ADHAM's interpretability at the population, subgroup, and individual levels. Extensive experiments on real-world datasets show that ADHAM provides novel insights into the association between exposures and outcomes. Further, ADHAM remains on par with existing state-of-the-art survival baselines in terms of predictive performance, offering a scalable and interpretable approach to time-to-event prediction in healthcare.
\end{abstract}

\section{Introduction}

Survival analysis, a subfield of machine learning (ML) and statistics, focuses on modeling time-to-event outcomes (\eg disease progression, hospital readmission, or relapse). In medical settings, the event of interest is often not observed for all patients due to study end, loss to follow-up, or withdrawal, resulting in censored data. This characteristic sets survival analysis apart from regression~\citep{clark2003survival, singh2011survival}. To address this challenge, survival analysis leverages the survival function $S(t \mid \xb) = p(T > t \mid \xb)$---which corresponds to the probability that an individual does not experience the event of interest past time $t$, for model learning~\citep{haider2020effective}. Survival analysis plays a critical role in healthcare, supporting both clinical decision-making and outcome evaluation in trials~\citep{vigano2000survival, fleming2000survival, cole2001polychemotherapy, faucett2002survival, hagar2014survival, perotte2015risk, panahiazar2015using, morita2009combined}.

Despite growing interest in ML for healthcare, clinical adoption of ML survival models remains limited~\citep{abdullah2021review, lu2023importance}. A key limitation of existing methodologies is the lack of interpretability of model predictions~\citep{shortliffe2018clinical, tonekaboni2019clinicians}. Clinicians are unlikely to trust or act on model outputs without a clear understanding of how and why a prediction was made, especially when guiding high-stakes decisions such as individualized care or the development of treatment guidelines~\citep{amann2020explainability}. This challenge is further complicated in survival analysis, where interpretability must account not only for which covariates induce the predictions, but also for how their influence evolves. Unlike standard ML tasks, survival interpretability requires continuous-time explanations that show how risk changes longitudinally. 

The granularity of interpretability in clinical risk models can be broadly categorized into three levels~\citep{ahmad2018interpretable}: (1) Population-level interpretability, which captures how covariates influence outcomes across the entire cohort and helps identify global risk trends~\citep{lou2013accurate}; (2) Subgroup-level interpretability, which uncovers patterns within latent groups of patients who share similar progression profiles and helps with tailored interventions~\citep{bhavnani2022framework}; and (3) Individual-level interpretability, which provides personalized insights into how specific covariates relate to a given patient’s risk over time~\citep{krzyzinski2023survshap}. For clinical relevance, models should ideally provide explanations at the population, subgroup, and individual levels. Yet, methods that support all three levels remain largely unexplored.

The predominant class of interpretable models in survival analysis extends the classical Cox Proportional Hazards (CoxPH) framework~\citep{cox1972regression}, often through the incorporation of Generalized Additive Models (GAMs), a flexible approach that represents the hazard function as a sum of smooth, potentially nonlinear transformations of individual covariates~\citep{jiang2022coxnams, peroni2022extending}. However, a key limitation of additive modeling lies in concurvity~\citep{siems2023curve, kovacs2024feature}, a form of multicollinearity where correlated features obscure each other’s effects, entangling their individual contributions. This phenomenon can lead to explanations that misalign with established physiological mechanisms or clinical expectations~\citep{ramsay2003effect, tan2024association}. For instance, GAMs showed limited clinical interpretability when applied to the Glasgow Coma Score (GCS), as correlations between its components (e.g., verbal, eye response) can distort the model’s explanations~\citep{hegselmann2020evaluation}. In addition, GAMs do not provide subgroup-level interpretability, an essential feature for identifying clinically meaningful heterogeneity and informing treatment guidelines~\citep{bhavnani2022framework}.

To overcome these challenges, we introduce Deep Additive Hazard Mixtures (ADHAM), which combines additive hazard modeling with mixture density networks. ADHAM represents the hazard function as a mixture of subgroup-specific hazard functions. Unlike former additive models, which assume a uniform sum of covariate effects~\citep{jiang2022coxnams, peroni2022extending, xu2023coxnam}, ADHAM learns a weighted combination to combine the hazard associated with each additive component. Importantly, ADHAM decouples the learning of hazard functions from subgroup assignment weights and trains each additive term separately to break the pairwise covariate correlations to address the concurvity problem. This leads to population-level hazard shapes that purely capture the covariate-specific trends in data. 

A key challenge in mixture models is how to select the number of subgroups a priori. The true number of latent subgroups is typically unknown, which leads practitioners to often train multiple models under varying numbers of groups. We address this with a principled, post-training model selection strategy. This allows ADHAM to adapt its subgroup size \emph{a posteriori}, and eliminates the need to train multiple models. Our results show that ADHAM achieves performance comparable to state-of-the-art additive models while offering improved interpretability.

\vspace{2ex}
Our contributions are as follows:

\begin{enumerate}
\item \textbf{Multi-level interpretability:}  
ADHAM provides explanations at the population, subgroup, and individual-levels within a single survival modeling framework. 
\vspace{-0.0ex}
To the best of our knowledge, it is the only survival analysis model that can provide all three levels of explanation at once (Section \ref{sec:model}).
\item \textbf{Overcoming concurvity through decoupled training:} We propose a training strategy for ADHAM that separates the learning of covariate-specific hazard functions from subgroup assignment. This approach mitigates concurvity by ensuring each hazard component is learned independently. Empirically, this leads to stable and reliable interpretability across training runs (Section \ref{sec:param}).
\item \textbf{Efficient model selection:} ADHAM includes a principled \emph{post-training} method to estimate and reduce redundant subgroups. To the best of our knowledge, ADHAM is the only survival analysis method that avoids training multiple models and allows for model selection post-training (Section \ref{sec:refine}). 
\item \textbf{Competitive predictive performance:} ADHAM achieves predictive performance comparable to state-of-the-art interpretable survival models across multiple benchmarks (Section \ref{sec:results}).
\end{enumerate}

\subsection*{Generalizable Insights about Machine Learning in the Context of Healthcare}
In healthcare, ML models should not only exhibit strong predictive performance but also interpretability. This is especially true in survival analysis, where understanding the risk at a given time is as important as surfacing associated predictors with fidelity. Clinicians must see how the risk evolves, how different covariates contribute to the risk, and how these patterns vary across the population, within subgroups, and for individual patients. 
Most ML models fall short in these areas, either acting as black-boxes, tackling a subset of these interpretability levels, or offering explanations that break down when the covariates are correlated. 
Additionally, models that attempt to identify subgroups typically require the true number of subgroups to be specified in advance, which introduces a challenging model selection problem ---often requiring training multiple models, which places an extra burden on practitioners. To address these issues, we introduce ADHAM, a survival analysis model that enables clear explanations at multiple levels, breaks covariate correlation during training to prevent concurvity, and facilitates easy post-training model selection. 

\section{Related Work}

\paragraph{Time-to-event models.} 

Traditional non-parametric survival analysis techniques such as Kaplan-Meier and Nelson-Aalen estimators model population-level survival functions~\citep{kaplan1958nonparametric, nelson1969hazard, aalen1978nonparametric}. These methods do not consider individual patient covariates and thus do not provide insights on the impact of exposure on outcome~\citep{haider2020effective}. To address this limitation, the foundational semi-parametric Cox Proportional Hazards (CoxPH)~\citep{cox1972regression} models the shift induced by individual covariates on a non-parametric population survival. The estimation of the impact of covariates on outcomes offers tailored risk assessments. DeepSurv~\citep{katzman2018deepsurv} extends this method by replacing the parametric relation between covariates and outcomes with a neural network, resulting in improved predictive performance. 

However, to obtain a closed-form likelihood, these approaches rely on the assumption, known as the proportional hazards (PH) assumption~\citep{grambsch1994proportional, hess1995graphical}, that the hazard function of any patient (1) differs from the population average by a constant factor and (2) follows the same trajectory through time. Avoiding this assumption, Cox-Time uses time-dependent conditional neural hazard functions~\citep{kvamme2019time}. Despite its flexibility, this approach introduces bias in parameter estimation. Recently, Deep Hazard Analysis (DHA) introduced an unbiased estimator of the likelihood, allowing unbiased and flexible non-PH survival estimation~\citep{ketenci2023maximum}. 

In another attempt to avoid the PH assumption, previous works aimed to estimate the time-to-event distribution instead of the hazard function. For instance, DeepHit approaches survival as a classification approach by learning probability mass functions over discrete time intervals~\citep{lee2018deephit}. This discretisation may, however, result in inexact likelihood estimates. Alternatively, recent works have proposed monotonic neural networks for estimating this distribution without approximation or discretization of the likelihood~\citep{rindt2022survival, jeanselme2023neural}.

Closer to our work, Mixture Density Networks (MDNs) estimate the complete time-to-event distribution as a mixture of base distributions. MDNs define a mixture of flexible probability density functions where both the density and mixture weights are modeled by neural networks. Given enough mixture components, MDNs can model any probability density~\citep{bishop2006pattern}. Recent approaches to MDNs use neural networks for group and time-to-event distributions. For example, Deep Survival Machines (DSM) uses a mixture of parametric distributions (such as Log-normal and Weibull)~\citep{nagpal2021dsm}. Survival Mixture Density Networks (Survival MDNs) assume a mixture of Gaussian distributions and maximize the associated likelihood after marginalizing latent subgroup assignments~\citep{han2022survival}.

\paragraph{Interpretable time-to-event models.}
Existing interpretable survival models for clinical risk prediction integrate CoxPH's framework with Generalized Additive Models (GAMs). For example, CoxNAM and TimeNAM use Neural Additive Models (NAMs) to model the hazard rate in Cox-Time~\citep{kvamme2019time, agarwal2021neural, jiang2022coxnams, utkin2022survnam, peroni2022extending, xu2023coxnam}. These models offer interpretability by defining a separate per-covariate function. This allows clinicians to visualize how individual covariates contribute to risk at different time points. 

Other lines of work utilize post-hoc interpretability tools. As an example, the SurvLIME explanation method estimates the cumulative hazard function of a black-box model by fitting CoxPH's regression model~\citep{kovalev2020survlime}. SurvSHAP(t) uses KernelSHAP to decompose the survival function to its Shapley values at each time step $t$~\citep{lundberg2017unified, krzyzinski2023survshap}. While effective, these ad-hoc interpretability methods become impractical for large datasets and extended time horizons, as they must be re-executed for every covariate-time combination.

\section{ADHAM: Additive Deep Hazard Analysis Mixtures}\label{sec:adham}
\subsection{Background: Generalized Additive Models}

GAMs are a flexible class of statistical models that extend generalized linear models by allowing for non-linear relationships between the covariates, denoted by a $D$-dimensional vector $\xb \in \mathbb{R}^D$, and the outcome. A $k^\text{th}$ order GAM, $g$, is defined as:

\begin{equation}
   g(\xb) = \sum_{u \subseteq [D] \mid |u| \leq k} g_u(\xb_u),\label{eq:gam}
\end{equation}

where $[D] = \{1,2,\cdots,D\}$. Each $g_u(\xb_{u})$ denotes the contribution of a covariate subset $u$ to the overall prediction. In this work, we adopt the additive decomposition in Equation~\ref{eq:gam} to model the hazard function considering $k=1$. Unlike traditional GAMs, however, we introduce a weighted sum formulation in which the influence of each component $g_u$ is modulated by patient-specific subgroup assignments, thereby enabling subgroup-specific additive effects.

\subsection{Model}\label{sec:model}

Consider a time-to-event dataset with $N$ points, $\mathcal{D} = \{(\mathbf{x}_i, t_i, \delta_i)\}_{i=1}^N$, where $\mathbf{x}_i \in \mathbb{R}^D$ denotes the vector of covariates associated with patient $i$, $t_i \in \mathbb{R}^+$ is the recorded event or censoring time, and $\delta_i \in \{0,1\}$ indicates whether the event occurred ($\delta_i = 1$) or was right-censored ($\delta_i = 0$). With these notations, we assume a latent subgroup membership $c$ dependent on covariates $\xb$, where covariate-weights are group-specific. We describe the data-generating process in Figure \ref{fig:dgp}:

\vspace{5ex}
\begin{minipage}[t]{0.44\textwidth}
\begin{center}
\begin{tikzpicture}
    \node[draw, rectangle, minimum width=6cm, minimum height=2.8cm] (box) at (2.25,0.75) {};

    \node at (0.6,1.8) {$i=\{1,\cdots,N\}$};
    \node[draw, rectangle, minimum width=1cm, minimum height=1cm, fill=gray!30] (x) at (0,0) {$\xb_i$};
    \node[draw, circle] (c) at (1.5,0) {$c_i$};
    \node[draw, circle] (d) at (3,0) {$d_i$};
    \node[draw, circle, fill=gray!30] (t) at (4.5,0) {$t_i$};
    \node[draw, rectangle, minimum width=1cm, minimum height=1cm, fill=gray!30] (delta) at (4.5,1.5) {$\delta_i$};
    
    \draw[->] (x) -- (c);
    \draw[->] (c) -- (d);
    \draw[->] (d) -- (t);
    \draw[<-] (t) -- (delta);
    \draw[->] (x) to[bend left=40] (t);

    \node[circle, fill=black, inner sep=1.5pt, label=below:$\boldsymbol{\beta}$] (beta) at (3,-1) {};
    \node[circle, fill=black, inner sep=1.5pt, label=below:$\theta$] (theta) at (1.5,-1) {};
    \node[circle, fill=black, inner sep=1.5pt, label=below:$\Phi$] (phi) at (4.5,-1) {};
    \draw[->] (theta) -- (c);
    \draw[->] (beta) -- (d);
    \draw[->] (phi) -- (t);
\end{tikzpicture}
\end{center}
\end{minipage}%
\begin{minipage}[t]{0.55\textwidth}
\vspace{-25ex}

$\quad$ For $i = \{1, \cdots, N\}$:
\begin{itemize}
    
    
    
    \item Draw  time-to-event $t_i \sim \cp{t}{\xb_i, \delta_i; \theta, \boldsymbol{\beta}, \Phi}$
    
    with marginal hazard function
    
    $\hspace{-1ex}\sum_{c=1}^C\sum_{d=1}^D \cp{d}{c; \boldsymbol{\beta}}\cp{c}{\xb_i; \theta} \lambda(t \mid x_{id}; \phi_d)$ 
    
    where, $c \mid \xb_i \sim \Cat{f_\theta(\xb_i)}$, 
    
    $ d \mid c \sim \Cat{\boldsymbol{\beta}_c}$, and $\Phi=\{\phi_d\}_{d=1}^D$
\end{itemize}
\end{minipage}
\vspace{0.5ex}
  \begin{minipage}{\linewidth}
    \captionof{figure}{Plate notation and data generating process of ADHAM. $f_\theta$ is subgroup assignment network, $\beta_c$ are subgroup-specific feature importance values, and $ \lambda(t \mid x_{id}; \phi_d)$ is the population-level hazard curve.}
    \label{fig:dgp}
  \end{minipage}

The subgroup assignment network $f_\theta(\xb)$ and associated weight vector $\boldsymbol{\beta}_c$ lie on $C$- and $D$-dimensional simplexes, respectively\footnote{This is ensured by softmax function.}. 
Each patient is assigned to a latent subgroup $c$, with probability $f_{\theta_c}(\xb)$ characterized by weights on each covariate-specific hazard, $\beta_{dc}$. Marginal hazard is then defined as the group-specific weighted sum of covariate-specific hazard functions. We summarize all notations with a table in Appendix \ref{app:notation} and illustrate the methodology with a flow chart in Appendix \ref{app:flow}.

In the following paragraphs, we describe how each term in ADHAM's marginal hazard function links to different levels of interpretability.

 \begin{figure}[H]
\includegraphics[width=\linewidth]{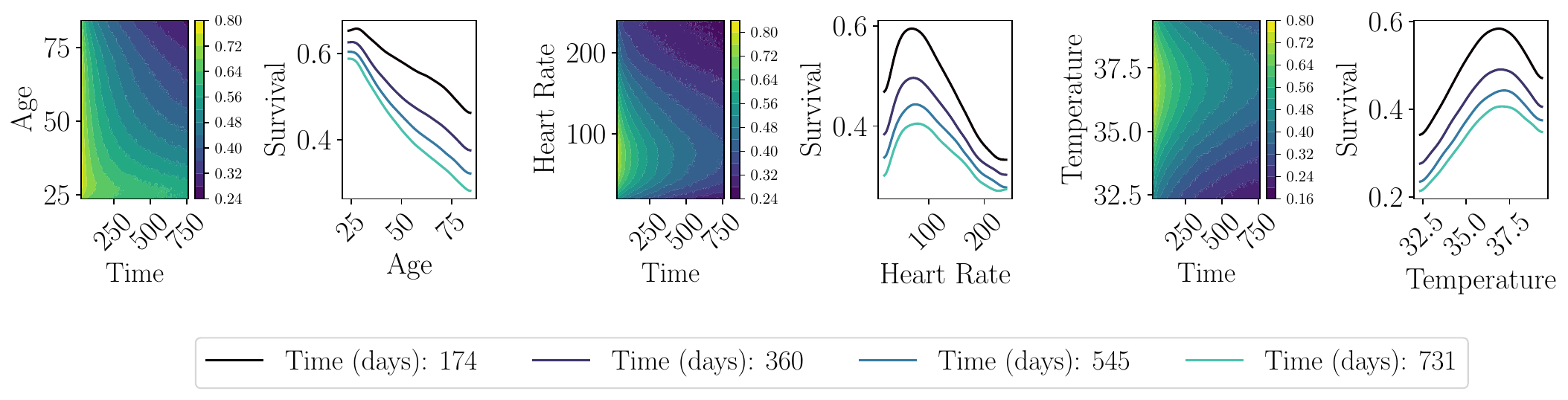}
\caption{Covariate-specific \emph{population-level} survival functions, $\lambda(t \mid x_{id}; \phi_d)$, of ADHAM trained on \textbf{SUPPORT} dataset. ADHAM captures well-known physiological trends in heart rate and temperature. In particular, normal ranges (\eg 36 - 37.5 $^{\circ}$C for temperature and 60 - 100 bpm for heart rate~\citep{tan2024association}) have better survival odds. Similarly, the survival probability decreases with age. In Appendix \ref{app:diffruns}, we compare ADHAM's population-level survival functions to those of TimeNAM's and demonstrate that our results are consistent. 
}\label{fig:motivation}
\end{figure}

\paragraph{Multi-level interpretability.} 


We start with the \emph{population-level} hazard, defined as:

\begin{equation} 
\lambda(t \mid x_d; \phi_d).\label{eq:pophaz} 
\end{equation}

This quantity represents the shared hazard as a function of the covariate $d$. It is the same across the population for the same $x_d$ values. One can either use this quantity directly or calculate the population-level survival function using it. In Figure \ref{fig:motivation}, we illustrate population-level survival functions, derived using Equation \eqref{eq:pophaz}, as also discussed in Appendix \ref{app:multilevel}.

We weigh each hazard function using subgroup-level weights $\boldsymbol{\beta} \in [0,1]^{C \times D}$, where $\sum_d \cp{d}{c; \boldsymbol{\beta}} = \sum_d \beta_{dc} = 1$. Note that, each $\boldsymbol{\beta}_c$ vector informs about \emph{subgroup-level} covariate importances (\eg if covariate $d$ is important for subgroup $c$, then $\beta_{dc}$ is high):

\begin{equation} 
\cp{d}{c; \boldsymbol{\beta}}\lambda(t \mid x_d; \phi_d) =  \beta_{dc} \lambda(t \mid x_d; \phi_d). \end{equation}

We use $\boldsymbol{\beta}$ for \emph{subgroup-level} interpretability and illustrate it, in Figure \ref{fig:clusters}, along with subgroup statistics. 

\begin{figure}[H]
\includegraphics[width=\linewidth]{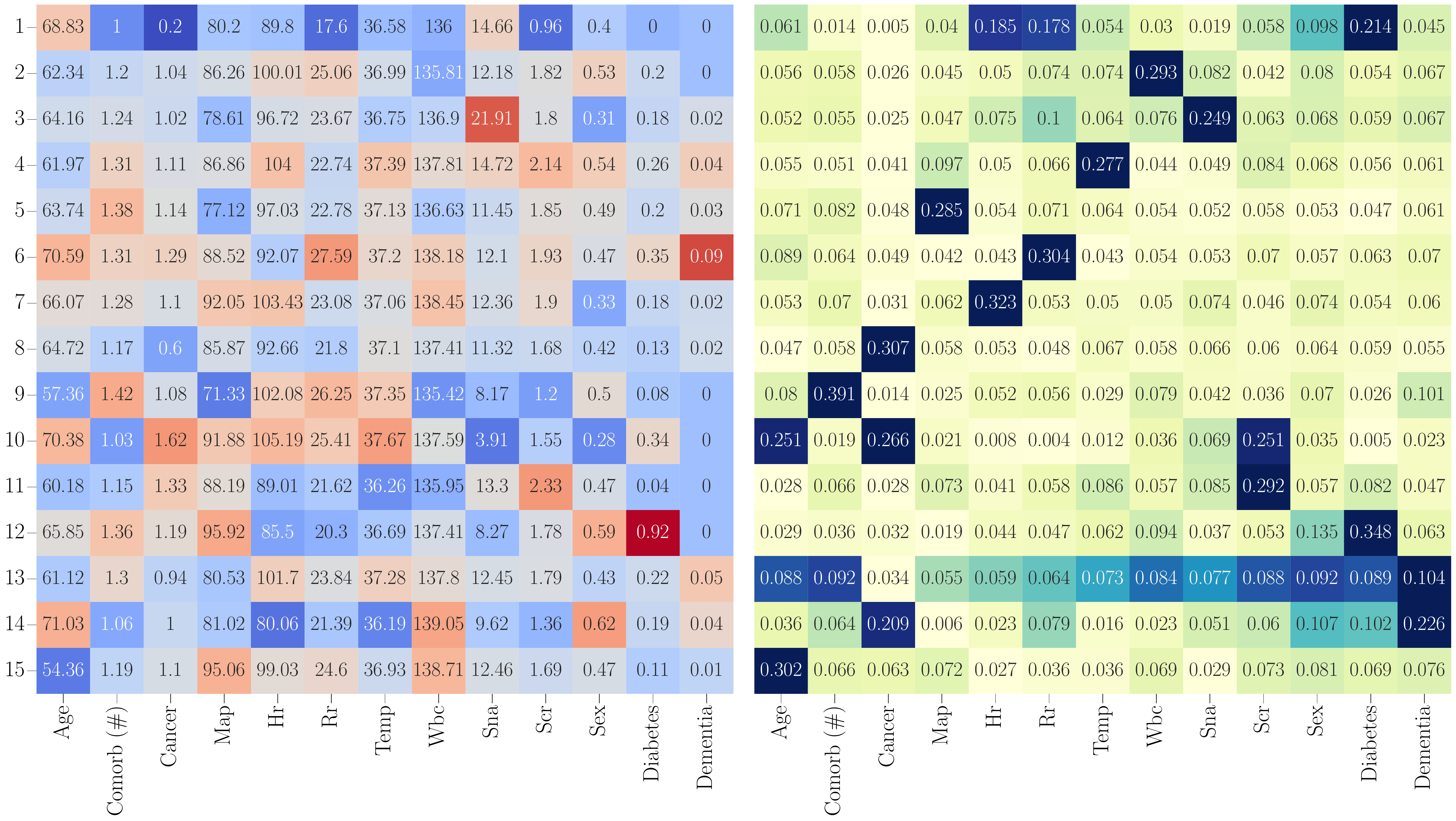}
    \caption{\emph{Subgroup-level} interpretability of \textbf{SUPPORT} dataset. Subgroup-specific average measurements and corresponding ${\beta}_{dc}$ values are provided in \textbf{left} and \textbf{right}, respectively. Each row describes a subgroup. In the left heatmap, \red{warmer} colors indicate values \underline{above population average}, while \blue{cooler} colors reflect \underline{below-average} measurements, with each cell including the corresponding numeric value. See Appendix~\ref{app:covariates} for covariate details. In the right heatmap, darker shades denote higher covariate importance, indicating which covariates drive the hazard in each subgroup. ADHAM consistently identifies covariates that deviate from normal ranges via $\boldsymbol{\beta}$ matrix. For instance, Subgroup 1 shows low heart and respiratory rates, both heavily weighted. Subgroup 2 is marked by low WBC and elevated heart and respiratory rates. Subgroup 4 is dominated by high temperature, while Subgroup 10 highlights age and cancer. Overall, ADHAM effectively identifies clinically relevant risk factors across subgroups.
}\label{fig:clusters}
\end{figure}

For a given patient $\xb$, we compute its group assignment via $\cp{c}{\xb; \theta}$ and define the marginal hazard rate as a sum over \emph{individual-level} hazard functions, each describing a contribution to the marginal patient hazard:

\begin{align} \lambda(t \mid \xb; \theta, \boldsymbol{\beta}, \Phi) &= \sum_{d=1}^D \sum_{c=1}^C {\cp{d}{c; \boldsymbol{\beta}}} \cp{c}{\xb; \theta} \lambda(t \mid x_d; \phi_d)  \\ &= \sum_{d=1}^D \underbrace{\cp{d}{\xb; \theta, \beta} \lambda(t \mid x_d; \phi_d)}_{\text{\emph{individual-level hazard function}}}\\
&= \sum_{d=1}^D \lambda(t \mid \xb; \theta, \boldsymbol{\beta}, \phi_d), \end{align}

where $\cp{d}{\xb; \theta, \beta} = \sum_{c=1}^C \cp{d}{c; \boldsymbol{\beta}}\cp{c}{\xb; \theta} = \sum_{c=1}^C \beta_{dc} f_{\theta_c}(\xb)$ is a function of all covariates $\xb$ and introduces patient-specific covariate interactions into the picture. Note that, $\sum_{c=1}^C f_{\theta_c}(\xb)=1$. We use $\lambda(t \mid \xb; \theta, \boldsymbol{\beta}, \phi_d)$ for \emph{individual-level} interpretability, as demonstrated in Figure \ref{fig:hazard}. 

\begin{figure}[!htb]
 \includegraphics[width=\linewidth]{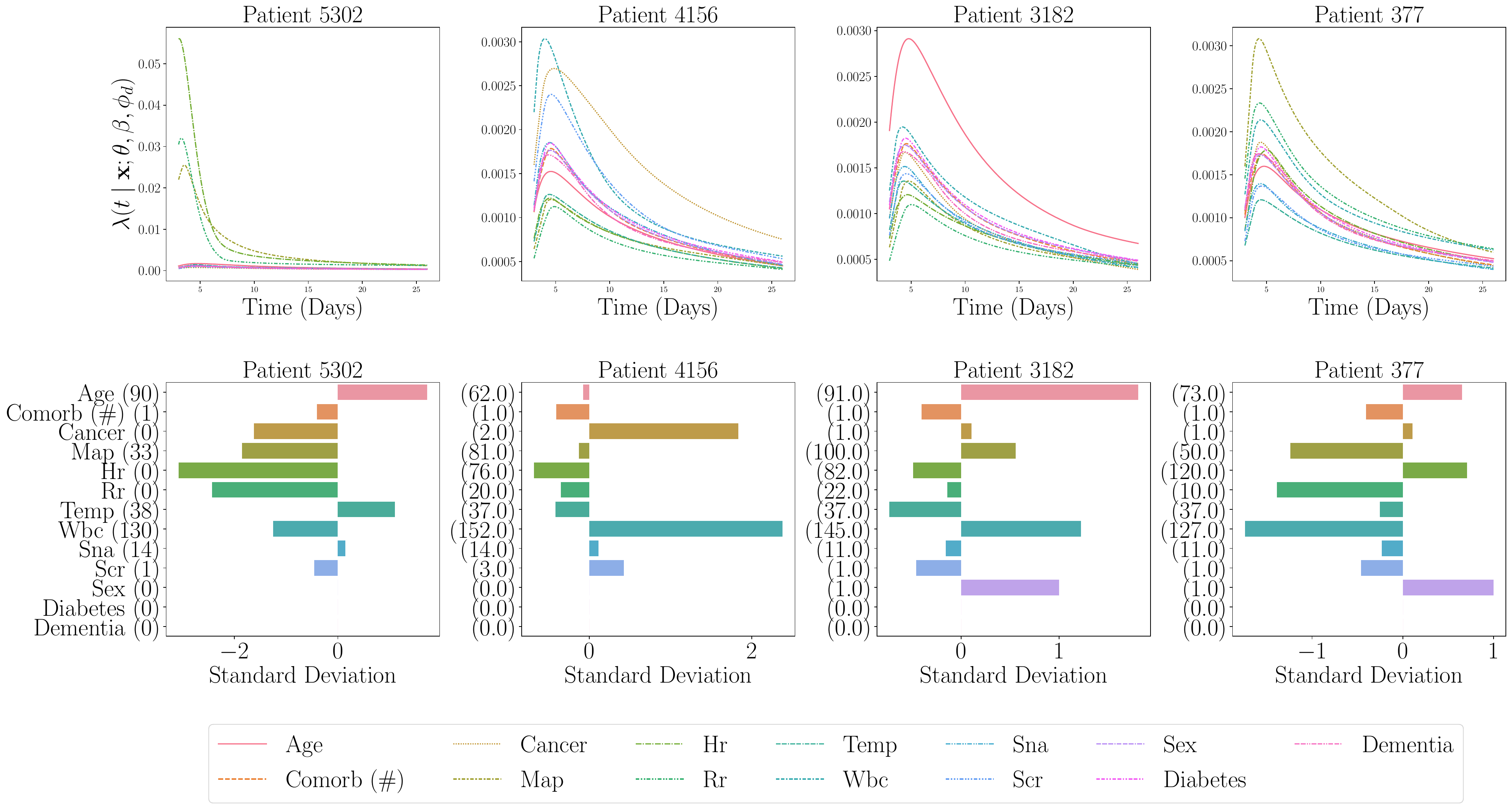}
    \caption{\emph{Individual-level} hazard functions and corresponding input values for four patients. For the first patient, ADHAM identifies an immediate risk driven by low heart rate (Hr), respiratory rate (Rr), and mean arterial pressure (Map), matching the patient’s high short-term risk. For Patient 4156, the model captures both short-term risk from elevated white blood cell (Wbc) count—suggesting possible infection—and long-term risk from advanced-stage cancer. Patient 3182, an older individual with moderately abnormal covariates, shows a hazard profile dominated by age. For Patient 377, ADHAM points to low Map and high Wbc as the key contributors to their risk, which are outside of normal range.
}\label{fig:hazard}
\end{figure}

In this section, we focused on the hazard function. Similarly, one can study the different levels of interpretability from the survival function perspective, as presented in Appendix \ref{app:multilevel}.

\subsection{Parameter Estimation}\label{sec:param}
Given the marginal patient hazard, the corresponding probability density is expressed as:

\begin{align} \cp{t}{\xb;\theta, \boldsymbol{\beta}, \Phi} = \lambda(t \mid \xb;\theta, \boldsymbol{\beta}, \Phi) \exp\left\{-\int_0^t \lambda(s \mid \xb; \theta, \boldsymbol{\beta}, \Phi), \diff{s}\right\}, \end{align}
and the corresponding log-likelihood over the dataset $\mathcal{D}$ is:
\begin{align} \ell(\Tb_N|\Xb_N, \Delta_N; \theta, \boldsymbol{\beta}, \Phi) = \log \prod_{i=1}^N \lambda(t_i \mid \xb_i; \theta, \boldsymbol{\beta}, \Phi)^{\delta_i} \exp\left\{-\int_0^{t_i} \lambda(s \mid \xb_i; \theta, \boldsymbol{\beta}, \Phi) \diff{s}\right\}, \label{eq:obj} \end{align}

where $\Tb_N = \{t_i\}_{i=1}^N$, $\Xb_N = \{\xb_i\}_{i=1}^N$, and $\Delta_N = \{\delta_i\}_{i=1}^N$. 

We use neural networks to model $f_\theta(\xb)$ and $\lambda(t \mid x_d, \phi_d)$. We ensure positive $\lambda(t \mid x_d, \phi_d)$ by applying the $\mathrm{softplus}(.)$ function.
A computationally efficient and unbiased stochastic approximation to Equation~\eqref{eq:obj}, using $L$ data samples and $M$ importance samples over time-to-event outcomes, has been proposed by~\cite{ketenci2023maximum}:

\begin{align}
\begin{split}\label{eq:mcloss}
    &\tilde{\ell} (\Tb_L|\Xb_L, \Delta_L, \tilde{\Tb}_{LM}; \theta, \boldsymbol{\beta}, \Phi) = \frac{N}{L} \sum_{i=1}^L \left(\delta_i \log \lambda(t_i \mid \xb_i ;  \theta, \boldsymbol{\beta}, \Phi)  - \frac{t_i}{M} \sum_{j=1}^M \lambda(\tilde{t}_{ij} \mid \xb_i ;  \theta, \boldsymbol{\beta}, \Phi)  \right)\nonumber,
\end{split}
\end{align}

where  $\tilde{\Tb}_{LM} = \{\{\tilde{t}_{ij} \}_{j=1}^M\}_{i=1}^L$, and $\tilde{t}_{ij} 
 \sim U(0, t_i)$. Unlike former numerical integration-based estimations, such as quadrature, this Monte Carlo (MC) estimate is unbiased~\citep{gregoire1995sampling}. 

\paragraph{Overcoming concurvity through decoupled training.}
Empirically, we observe that jointly optimizing $\{\theta, \boldsymbol{\beta}, \Phi\}$ can hinder interpretability due to two key factors: (1) concurvity, where correlated covariates distort each other’s contributions, leading to unreliable and unstable population-level hazard functions; and (2) ill-conditioning, where the flexibility of the subgroup assignment network allows individual-level hazard functions, $\lambda(t \mid \xb; \theta, \boldsymbol{\beta}, \phi_d)$, to behave as arbitrary, unconstrained functions over all covariates—particularly as the number of subgroups $C$ increases. Together, these effects would reduce the model’s interpretability.

Our goal is to ensure that ADHAM remains interpretable and stable regardless of concurvity and choice of $C$. To mitigate this issue, we disentangle the learning of covariate-specific hazard functions from the identification of subgroups as described in Algorithm \ref{alg:lrn}. In particular, we train each covariate-specific population-level hazard function and mixture assignment network separately by maximizing:

\begin{equation}
\tilde{\ell}_d (\Tb_L|\Xb_L, \Delta, \tilde{\Tb}_{LM}; \phi_d) = \frac{N}{L} \sum_{i=1}^L \left(\delta_i \log \lambda(t_i \mid x_{id} ; \phi_d)  - \frac{t_i}{M} \sum_{j=1}^M \lambda(\tilde{t}_{ij} \mid x_{id} ; \phi_d)  \right),
\end{equation}

with respect to $\phi_d$, and $\tilde{\ell} (\Tb_L|\Xb_L, \Delta_L, \tilde{\Tb}_{LM}; \theta, \boldsymbol{\beta}, \Phi)$ with respect to $\{\theta, \beta\}$. Independently maximizing the population-level hazard log-likelihood $\tilde{\ell}_d$ for each $d \in \{1, 2, \cdots, D\}$ helps eliminate covariate correlation during training by ensuring that each population-level hazard function is learned in isolation. Subsequently, optimizing the overall log-likelihood $\tilde{\ell} (\Tb_L|\Xb_L, \Delta_L, \tilde{\Tb}_{LM}; \theta, \boldsymbol{\beta}, \Phi)$ with respect to ${\theta, \boldsymbol{\beta}}$ learns a distribution $\cp{d}{\xb; \theta, \boldsymbol{\beta}}$ that re-weights the fixed hazard curves to explain the data likelihood best. This two-step approach prevents the mixing weights from influencing the population-level hazard functions, avoiding covariate-specific curves from degenerating into arbitrary functions of all covariates. The full training procedure is outlined in Algorithm~\ref{alg:lrn}.


\begin{algorithm}[h]
    \caption{Mini-batch stochastic gradient descent learning of ADHAM parameters. We use overall model log-likelihood as our early stopping criteria.}\label{alg:lrn}
    
\begin{algorithmic}

    \STATE \textbf{def} $\texttt{fit\_adham}(\mathcal{D}, \theta, \boldsymbol{\beta}, \Phi)$:

    \STATE $\quad$ $\theta, \boldsymbol{\beta}, \Phi \leftarrow{}$ Initialize neural network parameters
    
    \STATE $\quad$ \textbf{while} {$\tilde{\ell} (\Tb_L|\Xb_L, \Delta_L, \tilde{\Tb}_{LM}; \theta, \boldsymbol{\beta}, \Phi)$ has not converged} \textbf{do}

    \STATE $\quad$ $\quad$ $\{\Xb_L, \Tb_L, \Delta_L \}\leftarrow{}$ Sample $L$ data from $\mathcal{D}$
    
    \STATE $\quad$ $\quad$ $\tilde{\Tb}_{LM} \leftarrow{}$ Sample $M$ importance samples from $U(0, \Tb)$
    
    $\quad$ $\quad$ \# 1. Fit Hazard Functions

    \STATE $\quad$ $\quad$ \textbf{for} {$d = \{1,2,\cdots, D\}$} \textbf{do}
    
    \STATE  $\quad$ $\quad$ $\quad$ $g_d$ $\leftarrow{}$ $\nabla_{\phi_d }\tilde{\ell}_d (\Tb_L|\Xb_L, \Delta, \tilde{\Tb}_{LM}; \phi_d) $

    \STATE $\quad$ $\quad$ \textbf{end for}

    \STATE $\quad$ $\quad$ $\Phi$ $\leftarrow{}$ Update using gradients $\{g_d\}_{d=1}^D$ 

    $\quad$ $\quad$ \# 2. Fit Mixture Components
    
    \STATE $\quad$ $\quad$ ${g}$ $\leftarrow{}$ $\nabla_{\theta,\beta }\tilde{\ell} (\Tb_L|\Xb_L, \Delta_L, \tilde{\Tb}_{LM}; \theta, \boldsymbol{\beta}, \Phi)$
 
    \STATE $\quad$ $\quad$ $\theta, \beta$ $\leftarrow{}$ Update using gradients ${g}$ 

    \STATE $\quad$ \textbf{end while}

    \STATE $\quad$ {\bfseries Output:}  $\theta, \boldsymbol{\beta}, \Phi$
    
\end{algorithmic}
\end{algorithm}


\paragraph{Regularization of ADHAM.} To (i) encourage sparsity in subgroup assignments and (ii) promote broader exploration of covariate relevance during training, we subtract two regularization terms from the objective function $\tilde{l}$. While sparsity in covariate weights may eventually reflect meaningful signals, regularization ensures that this structure emerges through learning rather than an early optimization bias:

\begin{align}
    &\tilde{\mathcal{R}}(\Xb_L; \theta, \beta)\nonumber\\
    & =  \frac{1}{L(L-1)}\sum_{i=1}^L\sum_{\substack{j=1 \\ j\neq i}}^L\sum_{c=1}^C \cp{c}{\xb_i; \theta} \cp{c}{\xb_{j}; \theta}  +  \frac{1}{L}\sum_{i=1}^L \sum_{d=1}^D\cp{d}{\xb_i; \theta, \beta} \log \cp{d}{\xb_i; \theta, \beta}. 
\end{align}

The first term is known as an orthogonal output regularization~\citep{brock2016neural, bansal2018can}, which promotes diversity in subgroup assignments by encouraging the model to assign different data points to distinct subgroups with high confidence. The second term is known as an entropy-regularization term~\citep{mnih2016asynchronous, tang2023all}, which encourages ADHAM to consider a wider set of covariates during training, \emph{helping to avoid premature convergence} to narrow, locally optimal solutions. We conduct experiments both with and without incorporating regularization terms into the ADHAM objective, to study the effect empirically, as discussed in Section \ref{sec:results}.

\subsection{Model Selection via Subgroup Refinement}\label{sec:refine}
Determining the optimal number of heterogeneous subgroups is often difficult and usually involves training and evaluating several models. ADHAM circumvents this challenge through a post-training refinement strategy that takes advantage of two key properties: (1) each population-level hazard function is tied to a single covariate $x_d$, and (2) its contribution is scaled by a subgroup-specific constant $\beta_{dc}$.

\begin{algorithm}[h]
    \caption{Pseudo-algorithm for model selection given the subgroup groups $\mathcal{C}$, covariate importance matrix $\boldsymbol{\beta}$, and a threshold, $h$. In practice, we use efficient implementations provided by \texttt{fcluster} and \texttt{linkage} modules of SciPY~\citep{virtanen2020scipy}.}\label{alg:refine}
\begin{algorithmic}
\STATE \textbf{def} $\texttt{combine\_clusters}(\mathcal{C}, \boldsymbol{\beta}, h)$:

    \STATE $\quad$  $\mathcal{C}^\star \leftarrow{} \text{Initialize empty set } \{\}$

    \STATE $\quad$ $\rho$ $\leftarrow{}$ Initialize correlation matrix  $\frac{\beta \beta^\top}{\sqrt{\left(\tr{\beta \beta^\top}\right)\left(\tr{\beta \beta^\top}\right)^\top}}$ \# $C \times C$ correlation matrix
    
    \STATE $\quad$ \textbf{while} $\max \rho \geq h$ \textbf{do} \# While there are no correlated subgroups left

    \STATE  $\quad$  $\quad$ \textbf{for} {$c \in \mathcal{C}$} \textbf{do} 

    \STATE $\quad$ $\quad$ $\quad$\textbf{if} $\max \rho_{c} > h$ \textbf{then}
    
    \STATE $\quad$ $\quad$ $\quad$ $\quad$ $c^\star \leftarrow \argmax \rho_{c}$  \# $\rho_c$ is $C \times 1$ row vector

    \STATE $\quad$ $\quad$ $\quad$ $\quad$ $\mathcal{C}^\star \leftarrow$ Combine groups $ \mathcal{C}^\star \cup \{c, c^\star\}$  \# $ \mathcal{C}^\star$ is a set of (2 cardinality) sets
    
    \STATE $\quad$ $\quad$ $\quad$ $\quad$ $\rho_{c c^\star} \leftarrow$ Update entry to $- \infty$ \# So that \textbf{while} does not run forever

    \STATE $\quad$ $\quad$ $\quad$\textbf{end if}
    
    \STATE $\quad$ $\quad$ \textbf{end for}
    
    \STATE $\quad$ \textbf{end while}

    \STATE $\quad$ \textbf{while}  $\bigcap_{c^\star\in \mathcal{C}^\star}c^\star \neq \{\}$ \textbf{do} \# While there are groups to be merged 

    \STATE $\quad$ $\quad$ \textbf{for}  {$\{c^\star_1, c^\star_2\} \in \mathcal{C}^\star \times \mathcal{C}^\star$} \textbf{do} 

    \STATE $\quad$ $\quad$ $\quad$ \textbf{if} $c^\star_1 \cap c^\star_2 \neq \{\}$ \textbf{then} \# If there are common elements then merge
    
    \STATE $\quad$ $\quad$ $\quad$ $\quad$ $\mathcal{C^\star}$ $\leftarrow{}$ $\mathcal{C^\star} \setminus c_1^\star$ \# Subtract $c_1^\star$ 

    \STATE $\quad$ $\quad$ $\quad$ $\quad$ $\mathcal{C^\star}$ $\leftarrow{}$ $\mathcal{C^\star} \setminus c_2^\star$ \# Subtract $c_2^\star$

    \STATE $\quad$ $\quad$ $\quad$ $\quad$ $\mathcal{C^\star}$ $\leftarrow{}$ $\mathcal{C^\star} \cup \{c_1^\star \cup c_2^\star\}$ \# Add their union back

    \STATE $\quad$ $\quad$ $\quad$ \textbf{end if} 
    
    \STATE $\quad$ $\quad$ \textbf{end for} 
    
    \STATE $\quad$ \textbf{end while} 

   \STATE { $\quad$ \bfseries Output:}  $\mathcal{C}^\star$ \# Refined subgroups
\end{algorithmic}
\end{algorithm}

\begin{figure}[!htb]
    \raggedleft
    \begin{minipage}{0.5\textwidth}
        \raggedleft
        \includegraphics[width=\linewidth]{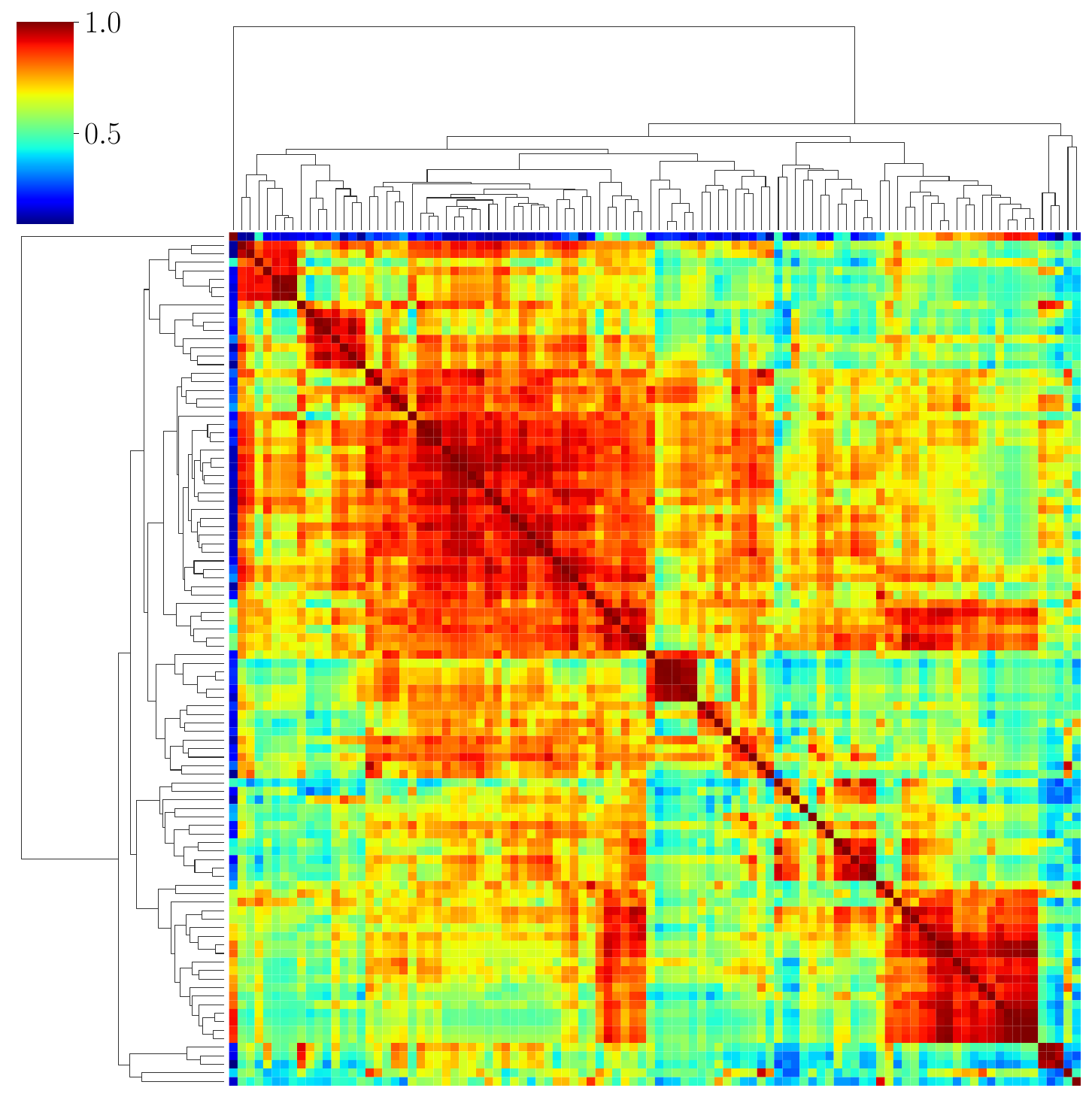}
        \centering
        \label{fig:corr1}
    \end{minipage}\hfil
    \begin{minipage}{0.5\textwidth}
    \raggedright
        \includegraphics[width=\linewidth]{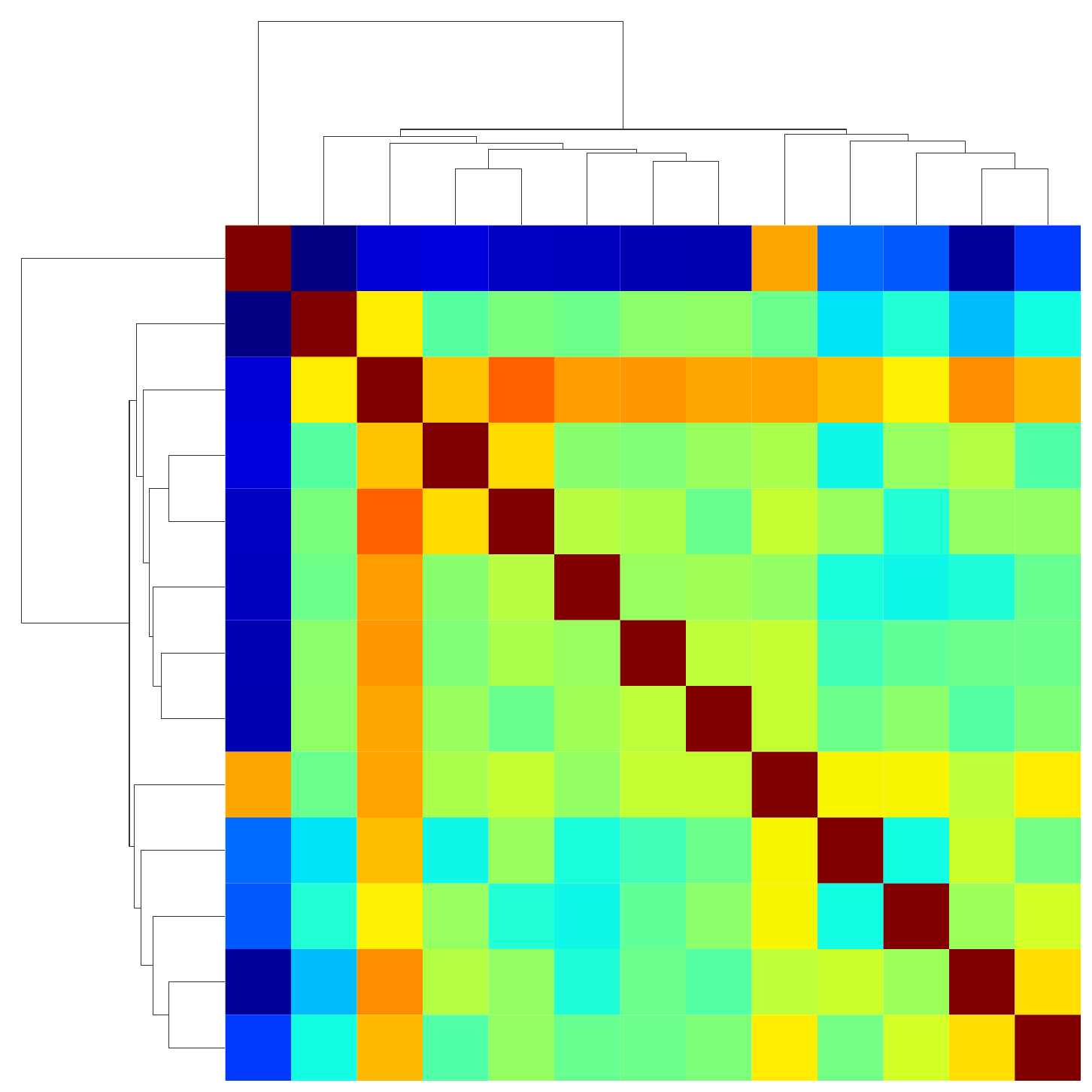}
        \centering
    \end{minipage}
    \caption{ Example correlation matrix $\rho$ before (left) and after (right) subgroup refinement on the \textbf{SUPPORT} dataset. Rows and columns are ordered using \texttt{sns.clustermap} to group similar subgroups. Each cell shows the correlation between two groups (e.g., $\rho_{c_1 c_2}$), with warmer colors indicating higher similarity. Before refinement, many subgroups have overlapping covariate importance patterns, suggesting repetition. After refinement, correlations decrease and subgroup profiles become more distinct, each emphasizing different sets of covariates.}\label{fig:rho}
\end{figure}
As such, if two subgroups $c_1$ and $c_2$ have identical covariate importance vectors, i.e., $\boldsymbol{\beta}_{c_1} = \boldsymbol{\beta}_{c_2}$, they can be combined without affecting the model’s data log-likelihood and overall performance. A formal proof is provided in Appendix~\ref{app:proof}. Note that, this property is not a byproduct of the parameter estimation procedure but the model design. In practice, we evaluate the similarity between subgroups by calculating their pairwise correlation: $\rho_{c_1 c_2} = \frac{\sum_{d=1}^D\beta_{c_1 d} \beta_{c_2 d}}{\sqrt{\left(\sum_{d=1}^D\beta_{c_1d}^2\right) \left(\sum_{d=1}^D\beta_{c_2d}^2\right)}}$. 
subgroups with correlation above a predefined threshold $h$ are then merged using a bottom-up agglomerative strategy~\citep{mullner2011modern}, as described in Algorithm~\ref{alg:refine}. We refer to this procedure as \emph{subgroup refinement}. We illustrate an example correlation matrix $\rho$ in Figure \ref{fig:rho}. In practice, ADHAM can initially be trained with a large number of subgroups $C$, which can later be reduced through this refinement process if needed. We demonstrate an example on \textbf{SUPPORT} dataset in Table \ref{tab:sup}.
\subsection{Predictions}

Making survival predictions with ADHAM requires calculating the following estimate: 

\begin{align}
      S(t | \xb; \theta, \boldsymbol{\beta}, \Phi) = \exp\left\{-\int_0^t \lambda(s \mid \xb; \theta, \boldsymbol{\beta}, \Phi), \diff{s}\right\}\label{eq:pred}
\end{align}

This integral can be calculated by Monte Carlo samples, or numerical integration~\citep{ketenci2023maximum, kvamme2019time}. In our experiments, we use the estimation method proposed by~\cite{ketenci2023maximum}.
\section{Experimental Setup}\label{sec:exp}
\subsection{Datasets} \label{subsection:datasets}
We experiment on two standard benchmark datasets, \textbf{SUPPORT} and \textbf{FLCHAIN}, widely used in survival analysis, and a real-world clinical dataset of patients with chronic kidney disease (\textbf{CKD}), tracking time to acute kidney injury based on electronic health records. Please see Appendix \ref{app:dataset} for a detailed description of the datasets.
\subsection{Baseline Models} \label{subsection:baselines}
We consider ten well-established baselines for survival analysis and state-of-the-art models: CoxPH~\citep{cox1972regression}, DeepSurv~\citep{katzman2018deepsurv}, RSF~\citep{ishwaran2008random}, Cox-Time~\citep{kvamme2019time}, TimeNAM \& TimeNA2M~\citep{peroni2022extending}, DSM~\citep{nagpal2021dsm}, DCM~\citep{nagpal2021dcm}, DeepHit~\citep{lee2018deephit}, and DHA~\citep{ketenci2023maximum}. Please see Appendix \ref{app:baselines} for a description of these models. \subsection{Evaluation Metrics} \label{subsection:metrics}
We follow the same evaluation setup as~\citet{li2023spatio, nagpal2022counterfactual, nagpal2021dsm, nagpal2021dcm, jeanselme2022neural, jeanselme2023neural, wang2022survtrace, lee2019temporal, ketenci2023maximum} and compare the average C-Index, Brier Score, and AUROC statistics over different time horizons to assess each model’s predictive performance. We assess the performance of ADHAM and baseline models using both discrimination and calibration performance metrics across three time horizons, as often used at deployment. The details of evaluation metrics are described in Appendix \ref{app:metric}.
\subsection{Empirical Setup}
We perform 5-fold cross-validation across all models and datasets. At each training step, we leave 20\% of the data out and divide the remaining data into training and validation sets by 70\% and 30\%, respectively. For a fair comparison, we use fixed random seeds. This ensures that the training, validation, and test sets seen by ADHAM and baseline models are identical. We standardize the datasets by subtracting the mean and dividing by the standard deviation of the covariates. We use Adam optimizer for neural models~\citep{kingma2014adam}. Each model is trained for 4000 epochs. The best model is saved based on the validation loss during training (early stopping), and evaluations are done over the held-out test set, which the models do not see during training. Hyperparameter settings are available at Appendix~\ref{app:hyper}.

\begin{table}[t!]
\center
\begin{adjustbox}{width=\textwidth}
\begin{tabular}{cccccccccc}
\hline
& \multicolumn{9}{c}{{SUPPORT} Dataset}\\ \hline
\multicolumn{1}{c|}{} & \multicolumn{3}{c|}{25$^\text{th}$ Quantile} & \multicolumn{3}{c|}{50$^\text{th}$ Quantile} & \multicolumn{3}{c}{75$^\text{th}$ Quantile} \\ \cline{2-10} 
\multicolumn{1}{c|}{\multirow{-2}{*}{Models}}& C-Index $\uparrow$ & BS $\downarrow$  & \multicolumn{1}{c|}{AUROC $\uparrow$} & C-Index $\uparrow$ & BS $\downarrow$  & \multicolumn{1}{c|}{AUROC $\uparrow$}  & C-Index $\uparrow$  & BS $\downarrow$    & AUROC $\uparrow$ \\\hline
\rowcolor{gray!15}
\multicolumn{1}{c|}{{ADHAM} ($\mathcal{R}$, $h=1$, $C=100$)  \;\;\;\;\;\;\;\;}& {{0.660}} & 0.144 & {{0.666}} & {{0.630}} & 0.222 & {0.644} & {{0.620}} & 0.247 & 0.658 \\\rowcolor{gray!15}
\multicolumn{1}{c|}{{ADHAM} ($\mathcal{R}$, $h=0.8$, $C=[24, 40]$)\;\;}& {0.660} &	0.145	& 0.666	&0.629&	0.223&	0.644&	 {{0.620}}&	0.247&	0.656
 \\\rowcolor{gray!15}
\multicolumn{1}{c|}{{ADHAM} ($\mathcal{R}$, $h=0.75$, $C=[17, 30]$)}& {0.660} &	0.145	& {0.666}	&0.629&	0.223&	0.644&	0.619 &	0.247&	0.656 \\\rowcolor{gray!15}
\multicolumn{1}{c|}{{ADHAM} ($\mathcal{R}$, $h=0.7$, $C=[14, 23]$)\;\;}& {0.660} &	0.145	& {0.666}	&0.629&	0.223&	0.644&	0.619 &	0.247&	0.656 \\\rowcolor{gray!15}
\multicolumn{1}{c|}{{ADHAM} ($\mathcal{R}$, $h=0.65$, $C=[11, 19]$)}& 0.659 &	0.145	& 0.665	&0.629&	0.223&	0.644&	0.619 &	0.247&	0.656 \\\rowcolor{gray!15}
\multicolumn{1}{c|}{{ADHAM}}& 0.613 & 0.142 & 0.616 & 0.588 & 0.219 & 0.602 & 0.588 & 0.237 & 0.639 \\
\hline                         
\end{tabular}
\end{adjustbox}
\vspace{0.2cm}
\caption{Performance of ADHAM on the {SUPPORT} dataset as a function of the subgroup refinement threshold $h$. Larger values of $h$ impose stricter criteria for subgroup merging—that is, merging occurs only when the corresponding $\beta_c$ parameters are identical. We observe that model performance remains largely stable across varying values of $h$, indicating robustness to the choice of refinement threshold.}
\label{tab:suph}
\end{table}

\begin{table}[t!]
\center
\begin{adjustbox}{width=\textwidth}
\begin{tabular}{cccccccccc}
\hline
& \multicolumn{9}{c}{\textbf{SUPPORT} Dataset}\\ \hline
\multicolumn{1}{c|}{} & \multicolumn{3}{c|}{25$^\text{th}$ Quantile} & \multicolumn{3}{c|}{50$^\text{th}$ Quantile} & \multicolumn{3}{c}{75$^\text{th}$ Quantile} \\ \cline{2-10} 
\multicolumn{1}{c|}{\multirow{-2}{*}{Models}}& C-Index $\uparrow$ & BS $\downarrow$  & \multicolumn{1}{c|}{AUROC $\uparrow$} & C-Index $\uparrow$ & BS $\downarrow$  & \multicolumn{1}{c|}{AUROC $\uparrow$}  & C-Index $\uparrow$  & BS $\downarrow$    & AUROC $\uparrow$ \\\hline
\multicolumn{1}{c|}{DeepSurv}& 0.605 & \textbf{0.143} & 0.608 & 0.598 & \textbf{0.217} & 0.618 & 0.609 & \textbf{0.232} & \textbf{0.659} \\
\multicolumn{1}{c|}{RSF}& \textbf{0.662} & \textbf{0.139} & \textbf{0.668} & \textbf{0.624} & \textbf{0.214} & \textbf{0.641} & \textbf{0.614} & \textbf{0.230}$^\star$ & \textbf{0.668}$^\star$ \\
\multicolumn{1}{c|}{DeepHit}& \textbf{0.652} & \textbf{0.141} & \textbf{0.659} & 0.613 & 0.220 & \textbf{0.633} & 0.605 & \textbf{0.236} & \textbf{0.661} \\
\multicolumn{1}{c|}{Cox-Time}& 0.625 & \textbf{0.141} & 0.631 & \textbf{0.614} & \textbf{0.214} & \textbf{0.635} & \textbf{0.616} & \textbf{0.230}$^\star$ & \textbf{0.663} \\
\multicolumn{1}{c|}{DSM}& 0.633 & \textbf{0.141} & 0.639 & \textbf{0.621} & \textbf{0.215} & \textbf{0.638} & 0.605 & 0.254 & 0.653 \\
\multicolumn{1}{c|}{DCM}& \textbf{0.657} & \textbf{0.138}$^\star$ & \textbf{0.663} & \textbf{0.620} & \textbf{0.213} & \textbf{0.638} & 0.603 & \textbf{0.234} & 0.652 \\
\multicolumn{1}{c|}{DHA}& \textbf{0.663}$^\star$ & \textbf{0.138}$^\star$ & \textbf{0.672}$^\star$ & \textbf{0.631}$^\star$ & \textbf{0.211}$^\star$ & \textbf{0.650}$^\star$ & \textbf{0.615} & \textbf{0.231} & \textbf{0.663} \\ \hline\rowcolor{gray!15}
\multicolumn{1}{c|}{CoxPH}& 0.553 & 0.262 & 0.558 & 0.567 & 0.222 & 0.588 & 0.590 & 0.351 & 0.649 \\\rowcolor{gray!15}
\multicolumn{1}{c|}{TIMENAM}& 0.621 & \textbf{0.142} & 0.627 & 0.612 & \textbf{\underline{0.215}} & 0.633 & \textbf{0.615} & \textbf{\underline{0.230}}$^\star$ & \textbf{\underline{0.666}} \\\rowcolor{gray!15}
\multicolumn{1}{c|}{TIMENA2M} & \textbf{0.643} & \textbf{\underline{0.139}} & 0.650 & \textbf{0.618} & \textbf{0.216} & \textbf{0.636} & \textbf{0.610} & 0.238 & 0.654 \\\rowcolor{gray!15}
\multicolumn{1}{c|}{\textbf{ADHAM} ($\mathcal{R}$)}& \textbf{\underline{0.660}} & \textbf{0.144} & \textbf{\underline{0.666}} & \textbf{\underline{0.630}} & 0.222 & \textbf{\underline{0.644}} & \textbf{\underline{0.620}}$^\star$ & 0.247 & \textbf{0.658} \\\rowcolor{gray!15}
\multicolumn{1}{c|}{\textbf{ADHAM}}& 0.613 & \textbf{0.142} & 0.616 & 0.588 & 0.219 & 0.602 & 0.588 & 0.237 & 0.639 \\
\hline                         
\end{tabular}
\end{adjustbox}
\vspace{0.2cm}
\caption{Results on the \textbf{SUPPORT} dataset. The models in gray region are interpretable. Best values are denoted by $^\star$, and results that are statistically close to the best values are shown in {\textbf{bold}}. Best performing interpretable values are \underline{underlined}. The standard error of the sample mean (SEM) values are provided in Table \ref{tab:superr}.}
\label{tab:sup}
\end{table}

\section{Results}\label{sec:results}

\paragraph{Interpretability of ADHAM.} We illustrate ADHAM’s \emph{population-level interpretability} on the \textbf{SUPPORT} dataset in Figure~\ref{fig:motivation}, showing that it reliably captures meaningful trends for key measurements. These patterns are a useful tool to understand and debug models~\citep{caruana2015intelligible}. 
\begin{table}[H]
\center
\begin{adjustbox}{width=\textwidth}
\begin{tabular}{cccccccccc}
\hline
& \multicolumn{9}{c}{\textbf{CKD} Dataset}\\ \hline
\multicolumn{1}{c|}{} & \multicolumn{3}{c|}{25$^\text{th}$ Quantile} & \multicolumn{3}{c|}{50$^\text{th}$ Quantile} & \multicolumn{3}{c}{75$^\text{th}$ Quantile} \\ \cline{2-10} 
\multicolumn{1}{c|}{\multirow{-2}{*}{Models}}& C-Index $\uparrow$ & BS $\downarrow$  & \multicolumn{1}{c|}{AUROC $\uparrow$} & C-Index $\uparrow$ & BS $\downarrow$  & \multicolumn{1}{c|}{AUROC $\uparrow$}  & C-Index $\uparrow$  & BS $\downarrow$    & AUROC $\uparrow$ \\
\hline
\multicolumn{1}{c|}{DeepSurv}& \textbf{0.620} & \textbf{0.089} & 0.635 & 0.611 & \textbf{0.164} & 0.653 & \textbf{0.608} & \textbf{0.217} & \textbf{0.684} \\
\multicolumn{1}{c|}{RSF}& \textbf{0.628} & \textbf{{0.088}}$^\star$ & \textbf{0.647} & \textbf{0.621} & \textbf{0.164} & \textbf{0.665} & \textbf{0.611} & \textbf{{0.216}}$^\star$ & \textbf{0.701} \\
\multicolumn{1}{c|}{DeepHit}& \textbf{{0.642}}$^\star$ & \textbf{{0.088}}$^\star$ & \textbf{{0.658}}$^\star$ & \textbf{{0.630}}$^\star$ & \textbf{{0.161}}$^\star$ & \textbf{{0.673}}$^\star$ & \textbf{{0.617}}$^\star$ & \textbf{{0.216}}$^\star$ & \textbf{0.698} \\
\multicolumn{1}{c|}{Cox-Time}& \textbf{0.615} & \textbf{0.090} & 0.626 & 0.611 & \textbf{0.165} & 0.651 & \textbf{0.609} & \textbf{0.219} & \textbf{0.684} \\
\multicolumn{1}{c|}{DSM}& \textbf{0.630} & \textbf{0.091} & \textbf{0.644} & \textbf{0.619} & 0.182 & \textbf{0.669} & \textbf{0.607} & 0.244 & \textbf{{0.702}}$^\star$ \\
\multicolumn{1}{c|}{DHA}& \textbf{0.623} & \textbf{0.089} & \textbf{0.636} & \textbf{0.609} & \textbf{0.164} & \textbf{0.651} & \textbf{0.606} & \textbf{0.217} & \textbf{0.690} \\\hline \rowcolor{gray!15}
\multicolumn{1}{c|}{CoxPH}& 0.593 & \textbf{0.089} & 0.608 & 0.582 & \textbf{0.168} & 0.615 & 0.579 & \textbf{0.224} & 0.658 \\\rowcolor{gray!15}
\multicolumn{1}{c|}{TIMENAM}& \textbf{0.622} & \textbf{\underline{0.089}} & \textbf{\underline{0.639}} & \textbf{0.610} & \textbf{0.168} & \textbf{0.651} & \textbf{0.597} &  \textbf{\underline{0.217}} &  \textbf{\underline{0.688}} \\\rowcolor{gray!15}
\multicolumn{1}{c|}{TIMENA2M}& \textbf{\underline{0.625}} & 0.099 & \textbf{\underline{0.639}} & {0.612} & 0.192 & 0.646 & 0.591 & 0.261 & 0.647 \\\rowcolor{gray!15}
\multicolumn{1}{c|}{\textbf{ADHAM} ($\mathcal{R}$)}& \textbf{0.620} & \textbf{0.090} & \textbf{0.632} & \textbf{\underline{0.618}} & \textbf{0.170} &  \textbf{\underline{0.652}} &  \textbf{\underline{0.607}} & 0.236 & \textbf{0.682}\\\rowcolor{gray!15}
\multicolumn{1}{c|}{\textbf{ADHAM}}& 0.603 & \textbf{0.089} & 0.615 & 0.597 & \textbf{\underline{0.167}} & 0.630 & 0.592 & \textbf{0.226} & 0.657 \\
\hline               
\end{tabular}
\end{adjustbox}
\vspace{0.2cm}
\caption{Results on the \textbf{CKD} dataset. The models in the gray region are interpretable. Best values are denoted by $^\star$, and results that are statistically close to the best values are shown in {\textbf{bold}}. Best performing interpretable values are \underline{underlined}. The standard error of the sample mean (SEM) values are provided in Table \ref{tab:ckderr}.}
\label{tab:ckd}
\end{table}

Figure~\ref{fig:clusters} illustrates ADHAM’s \emph{subgroup-level interpretability}, by uncovering heterogeneous patient subgroups along with their associated covariate importance profiles. We observe that ADHAM groups patients based on abnormal measurement patterns, which are captured through the subgroup-specific weight matrix $\boldsymbol{\beta}$. In Figure~\ref{fig:hazard}, we present patient-specific hazard functions, where we once again observe that ADHAM assigns importance to most critical covariates. 


\paragraph{Performance of ADHAM.} We present performance results in Tables \ref{tab:sup}, \ref{tab:ckd}, and \ref{tab:flc}. The best average results are denoted by $^\star$, values that are statistically close to the best results (with $p=0.05$) with respect to two-sided Welch's t-test are denoted with \textbf{bold}\footnote{Note that the difference in two results may be statistically insignificant due to (1) the mean over fold and (2) the (corrected) sample standard deviation of runs.}, and best interpretable methods are \underline{underlined}. DHA achieves the best performance across most evaluation metrics, though it functions as a black-box model. Among the interpretable approaches—CoxPH, TIMENAM, TIMENA2M, and ADHAM—ADHAM with regularization ranks highest in 9 metrics, matching TIMENA2M, followed by TIMENAM with 8 and CoxPH with 2, across various datasets. We also find that regularization benefits ADHAM, particularly in ranking-based metrics, albeit with a minor trade-off in calibration accuracy. Furthermore, in Table \ref{tab:suph}---on the \textbf{SUPPORT} dataset, we observe that ADHAM's model selection strategy, where we alter $h=1$ to $h=0.65$, results in a minor performance loss while substantially reducing the number of subgroups (\eg from $100$ to $11$). Overall, ADHAM demonstrates similar performance to the existing state-of-the-art interpretable survival models. To evaluate statistical significance, the standard errors are provided in Appendix \ref{app:err}.

\begin{table}[t!]
\center
\begin{adjustbox}{width=\textwidth}
\begin{tabular}{cccccccccc}
\hline
& \multicolumn{9}{c}{\textbf{FLCHAIN} Dataset}\\ \hline
\multicolumn{1}{c|}{} & \multicolumn{3}{c|}{25$^\text{th}$ Quantile} & \multicolumn{3}{c|}{50$^\text{th}$ Quantile} & \multicolumn{3}{c}{75$^\text{th}$ Quantile} \\ \cline{2-10} 
\multicolumn{1}{c|}{\multirow{-2}{*}{Models}}& C-Index $\uparrow$ & BS $\downarrow$  & \multicolumn{1}{c|}{AUROC $\uparrow$} & C-Index $\uparrow$ & BS $\downarrow$  & \multicolumn{1}{c|}{AUROC $\uparrow$}  & C-Index $\uparrow$  & BS $\downarrow$    & AUROC $\uparrow$ \\
\hline
\multicolumn{1}{c|}{DeepSurv}& \textbf{0.788} & \textbf{0.060} & \textbf{0.799} & \textbf{0.792} & \textbf{0.100} & \textbf{0.815} & \textbf{0.788} & \textbf{0.126} & \textbf{0.823} \\
\multicolumn{1}{c|}{RSF}& \textbf{0.801} & \textbf{0.059} & \textbf{0.812} & \textbf{0.798} & \textbf{0.099} & \textbf{0.821} & \textbf{0.793} & \textbf{0.124} & \textbf{0.828} \\
\multicolumn{1}{c|}{DeepHit}& \textbf{0.794} & \textbf{0.061} & \textbf{0.805} & \textbf{0.795} & \textbf{0.101} & \textbf{0.818} & \textbf{0.792} & \textbf{0.127} & \textbf{0.827} \\
\multicolumn{1}{c|}{Cox-Time}& \textbf{0.798} & \textbf{0.065} & \textbf{0.810} & \textbf{0.797} & 0.118 & \textbf{0.820} & \textbf{0.793} & 0.160 & \textbf{0.828} \\
\multicolumn{1}{c|}{DSM}& 0.760 & 0.065 & 0.770 & 0.766 & 0.118 & 0.786 & 0.768 & 0.159 & 0.800 \\
\multicolumn{1}{c|}{DCM}& \textbf{0.790} & \textbf{0.059} & \textbf{0.801} & \textbf{0.788} & \textbf{0.100} & 0.810 & \textbf{0.782} & \textbf{0.128} & 0.814 \\
\multicolumn{1}{c|}{DHA}& \textbf{0.801} & \textbf{{0.058}}$^\star$ & \textbf{0.813} & \textbf{{0.800}}$^\star$ & \textbf{{0.097}}$^\star$ & \textbf{{0.824}}$^\star$ & \textbf{{0.795}}$^\star$ & \textbf{{0.122}}$^\star$ & \textbf{{0.830}}$^\star$ \\ \hline\rowcolor{gray!15}
\multicolumn{1}{c|}{CoxPH}& \textbf{0.789} & 0.103 & \textbf{0.801} & \textbf{0.794} & \underline{0.098} & \textbf{0.817} & \underline{\textbf{0.791}} & 0.168 & \textbf{0.826} \\\rowcolor{gray!15}
\multicolumn{1}{c|}{TIMENAM}& \textbf{0.796} & \textbf{\underline{0.060}} & \textbf{0.807} & \textbf{0.796} & \textbf{0.104} & \textbf{0.818} & \underline{\textbf{0.791}} & \underline{\textbf{0.134}} &\textbf{0.826} \\\rowcolor{gray!15}
\multicolumn{1}{c|}{TIMENA2M}& \textbf{\underline{0.802}}$^\star$ & 0.065 & \textbf{\underline{0.814}}$^\star$ & \textbf{\underline{0.797}} & 0.118 & \underline{\textbf{0.819}} & \underline{\textbf{0.791}} & 0.160 & \underline{\textbf{0.828}} \\\rowcolor{gray!15}
\multicolumn{1}{c|}{\textbf{ADHAM} ($\mathcal{R}$)}& \textbf{0.798} & \textbf{0.064} & \textbf{0.809} & \textbf{0.795} & 0.114 & \textbf{0.818} & \underline{\textbf{0.791}} & 0.154 & \textbf{0.825} \\\rowcolor{gray!15}
\multicolumn{1}{c|}{\textbf{ADHAM}}& 0.776 & \textbf{\underline{0.060}} & 0.786 & 0.782 & \textbf{0.102} & 0.804 & 0.779 & 0.135 & 0.810 \\
\hline            
\end{tabular}
\end{adjustbox}
\vspace{0.2cm}
\caption{Results on the \textbf{FLCHAIN} dataset. The models in the gray region are interpretable. Best values are denoted by $^\star$, and results that are statistically close to the best values are shown in {\textbf{bold}}. Best performing interpretable values are \underline{underlined}.The standard error of the sample mean (SEM) values are provided in Table \ref{tab:flcerr}.}
\label{tab:flc}
\end{table}
\section{Discussion and Limitations}\label{sec:lim}

ADHAM offers a unified framework for interpretability, delivering individualized, subgroup-level, and population-wide risk explanations. Our results show that ADHAM performs on par with leading interpretable survival models. Nonetheless, several limitations remain.

\paragraph{Causal interpretation.}
ADHAM is a predictive, not a causal model. While the proposed tool provides insights into the learned relation between covariates and survival outcomes, practitioners should not interpret these observational correlations as causation. 

 \paragraph{Sensitivity to dataset biases.}
 As with any data-driven model, ADHAM reflects the characteristics and potential biases present in the training data. These biases may affect both predictions and interpretations.

\paragraph{Regularization trade-offs.}
While regularization improves ADHAM’s ranking performance, particularly in identifying high-risk patients, our results show a slight reduction in calibration accuracy. Future work could explore ways to balance this trade-off more effectively.

\paragraph{Applicability to other data types.}
ADHAM is tailored for structured, tabular, and time-series data. Adapting it to handle other modalities, such as imaging or text, is a promising area for future exploration.

\paragraph{Competing risks.} The current formulation of ADHAM is limited to single-risk scenarios, and adapting it to competing risk frameworks is an open problem with important ramifications~\citep{jeanselme2025}.

\section{Conclusion}
In this paper, we introduce ADHAM, a novel survival analysis model that integrates deep additive hazard functions with a mixture-based structure to provide interpretable predictions at the population, subgroup, and individual levels. By decoupling hazard and mixture learning, ADHAM mitigates common interpretability challenges such as concurvity and provides patient-specific risk explanations. Furthermore, we propose a post-training refinement mechanism for selecting the number of subgroups \emph{a posteriori}. ADHAM achieves competitive predictive performance while offering practical interpretability of covariates that support clinical understanding and decision-making, positioning it as a useful tool for real-world applications in healthcare.

 \acks{Mert Ketenci acknowledges this research is supported
by NHLBI award R01HL148248.}

\bibliography{sample}

\begin{thebibliography}{71}
\providecommand{\natexlab}[1]{#1}
\providecommand{\url}[1]{\texttt{#1}}
\expandafter\ifx\csname urlstyle\endcsname\relax
  \providecommand{\doi}[1]{doi: #1}\else
  \providecommand{\doi}{doi: \begingroup \urlstyle{rm}\Url}\fi

\bibitem[Aalen(1978)]{aalen1978nonparametric}
Odd Aalen.
\newblock Nonparametric inference for a family of counting processes.
\newblock \emph{The Annals of Statistics}, pages 701--726, 1978.

\bibitem[Abdullah et~al.(2021)Abdullah, Zahid, and Ali]{abdullah2021review}
Talal~AA Abdullah, Mohd Soperi~Mohd Zahid, and Waleed Ali.
\newblock A review of interpretable ml in healthcare: taxonomy, applications, challenges, and future directions.
\newblock \emph{Symmetry}, 13\penalty0 (12):\penalty0 2439, 2021.

\bibitem[Agarwal et~al.(2021)Agarwal, Melnick, Frosst, Zhang, Lengerich, Caruana, and Hinton]{agarwal2021neural}
Rishabh Agarwal, Levi Melnick, Nicholas Frosst, Xuezhou Zhang, Ben Lengerich, Rich Caruana, and Geoffrey~E Hinton.
\newblock Neural additive models: Interpretable machine learning with neural nets.
\newblock \emph{Advances in neural information processing systems}, 34:\penalty0 4699--4711, 2021.

\bibitem[Ahmad et~al.(2018)Ahmad, Eckert, and Teredesai]{ahmad2018interpretable}
Muhammad~Aurangzeb Ahmad, Carly Eckert, and Ankur Teredesai.
\newblock Interpretable machine learning in healthcare.
\newblock In \emph{Proceedings of the 2018 ACM international conference on bioinformatics, computational biology, and health informatics}, pages 559--560, 2018.

\bibitem[Amann et~al.(2020)Amann, Blasimme, Vayena, Frey, Madai, and Consortium]{amann2020explainability}
Julia Amann, Alessandro Blasimme, Effy Vayena, Dietmar Frey, Vince~I Madai, and Precise4Q Consortium.
\newblock Explainability for artificial intelligence in healthcare: a multidisciplinary perspective.
\newblock \emph{BMC medical informatics and decision making}, 20:\penalty0 1--9, 2020.

\bibitem[Bansal et~al.(2018)Bansal, Chen, and Wang]{bansal2018can}
Nitin Bansal, Xiaohan Chen, and Zhangyang Wang.
\newblock Can we gain more from orthogonality regularizations in training deep networks?
\newblock \emph{Advances in Neural Information Processing Systems}, 31, 2018.

\bibitem[Bhavnani et~al.(2022)Bhavnani, Zhang, Visweswaran, Raji, and Kuo]{bhavnani2022framework}
Suresh~K Bhavnani, Weibin Zhang, Shyam Visweswaran, Mukaila Raji, and Yong-Fang Kuo.
\newblock A framework for modeling and interpreting patient subgroups applied to hospital readmission: visual analytical approach.
\newblock \emph{JMIR Medical Informatics}, 10\penalty0 (12):\penalty0 e37239, 2022.

\bibitem[Bishop and Nasrabadi(2006)]{bishop2006pattern}
Christopher~M Bishop and Nasser~M Nasrabadi.
\newblock \emph{Pattern recognition and machine learning}, volume~4.
\newblock Springer, 2006.

\bibitem[Brier et~al.(1950)]{brier1950verification}
Glenn~W Brier et~al.
\newblock Verification of forecasts expressed in terms of probability.
\newblock \emph{Monthly weather review}, 78\penalty0 (1):\penalty0 1--3, 1950.

\bibitem[Brock et~al.(2016)Brock, Lim, Ritchie, and Weston]{brock2016neural}
Andrew Brock, Theodore Lim, James~M Ritchie, and Nick Weston.
\newblock Neural photo editing with introspective adversarial networks.
\newblock \emph{arXiv preprint arXiv:1609.07093}, 2016.

\bibitem[Caruana et~al.(2015)Caruana, Lou, Gehrke, Koch, Sturm, and Elhadad]{caruana2015intelligible}
Rich Caruana, Yin Lou, Johannes Gehrke, Paul Koch, Marc Sturm, and Noemie Elhadad.
\newblock Intelligible models for healthcare: Predicting pneumonia risk and hospital 30-day readmission.
\newblock In \emph{Proceedings of the 21th ACM SIGKDD international conference on knowledge discovery and data mining}, pages 1721--1730, 2015.

\bibitem[Clark et~al.(2003)Clark, Bradburn, Love, and Altman]{clark2003survival}
Taane~G Clark, Michael~J Bradburn, Sharon~B Love, and Douglas~G Altman.
\newblock Survival analysis part i: basic concepts and first analyses.
\newblock \emph{British journal of cancer}, 89\penalty0 (2):\penalty0 232--238, 2003.

\bibitem[Cole et~al.(2001)Cole, Gelber, Gelber, Coates, and Goldhirsch]{cole2001polychemotherapy}
Bernard~F Cole, Richard~D Gelber, Shari Gelber, Alan~S Coates, and Aron Goldhirsch.
\newblock Polychemotherapy for early breast cancer: an overview of the randomised clinical trials with quality-adjusted survival analysis.
\newblock \emph{The Lancet}, 358\penalty0 (9278):\penalty0 277--286, 2001.

\bibitem[Cox(1972)]{cox1972regression}
David~R Cox.
\newblock Regression models and life-tables.
\newblock \emph{Journal of the Royal Statistical Society: Series B (Methodological)}, 34\penalty0 (2):\penalty0 187--202, 1972.

\bibitem[Dispenzieri et~al.(2012)Dispenzieri, Katzmann, Kyle, Larson, Therneau, Colby, Clark, Mead, Kumar, Melton~III, et~al.]{dispenzieri2012use}
Angela Dispenzieri, Jerry~A Katzmann, Robert~A Kyle, Dirk~R Larson, Terry~M Therneau, Colin~L Colby, Raynell~J Clark, Graham~P Mead, Shaji Kumar, L~Joseph Melton~III, et~al.
\newblock Use of nonclonal serum immunoglobulin free light chains to predict overall survival in the general population.
\newblock In \emph{Mayo Clinic Proceedings}, volume~87, pages 517--523. Elsevier, 2012.

\bibitem[Faucett et~al.(2002)Faucett, Schenker, and Taylor]{faucett2002survival}
Cheryl~L Faucett, Nathaniel Schenker, and Jeremy~MG Taylor.
\newblock Survival analysis using auxiliary variables via multiple imputation, with application to aids clinical trial data.
\newblock \emph{Biometrics}, 58\penalty0 (1):\penalty0 37--47, 2002.

\bibitem[Fleming and Lin(2000)]{fleming2000survival}
Thomas~R Fleming and DY~Lin.
\newblock Survival analysis in clinical trials: past developments and future directions.
\newblock \emph{Biometrics}, 56\penalty0 (4):\penalty0 971--983, 2000.

\bibitem[Graf et~al.(1999)Graf, Schmoor, Sauerbrei, and Schumacher]{graf1999assessment}
Erika Graf, Claudia Schmoor, Willi Sauerbrei, and Martin Schumacher.
\newblock Assessment and comparison of prognostic classification schemes for survival data.
\newblock \emph{Statistics in medicine}, 18\penalty0 (17-18):\penalty0 2529--2545, 1999.

\bibitem[Grambsch and Therneau(1994)]{grambsch1994proportional}
Patricia~M Grambsch and Terry~M Therneau.
\newblock Proportional hazards tests and diagnostics based on weighted residuals.
\newblock \emph{Biometrika}, 81\penalty0 (3):\penalty0 515--526, 1994.

\bibitem[Gregoire and Valentine(1995)]{gregoire1995sampling}
Timothy~G Gregoire and Harry~T Valentine.
\newblock A sampling strategy to estimate the area and perimeter of irregularly shaped planar regions.
\newblock \emph{Forest science}, 41\penalty0 (3):\penalty0 470--476, 1995.

\bibitem[Hagar et~al.(2014)Hagar, Albers, Pivovarov, Chase, Dukic, and Elhadad]{hagar2014survival}
Yolanda Hagar, David Albers, Rimma Pivovarov, Herbert Chase, Vanja Dukic, and No{\'e}mie Elhadad.
\newblock Survival analysis with electronic health record data: Experiments with chronic kidney disease.
\newblock \emph{Statistical Analysis and Data Mining: The ASA Data Science Journal}, 7\penalty0 (5):\penalty0 385--403, 2014.

\bibitem[Haider et~al.(2020)Haider, Hoehn, Davis, and Greiner]{haider2020effective}
Humza Haider, Bret Hoehn, Sarah Davis, and Russell Greiner.
\newblock Effective ways to build and evaluate individual survival distributions.
\newblock \emph{J. Mach. Learn. Res.}, 21\penalty0 (85):\penalty0 1--63, 2020.

\bibitem[Han et~al.(2022)Han, Goldstein, and Ranganath]{han2022survival}
Xintian Han, Mark Goldstein, and Rajesh Ranganath.
\newblock Survival mixture density networks.
\newblock \emph{arXiv preprint arXiv:2208.10759}, 2022.

\bibitem[Hegselmann et~al.(2020)Hegselmann, Volkert, Ohlenburg, Gottschalk, Dugas, and Ertmer]{hegselmann2020evaluation}
Stefan Hegselmann, Thomas Volkert, Hendrik Ohlenburg, Antje Gottschalk, Martin Dugas, and Christian Ertmer.
\newblock An evaluation of the doctor-interpretability of generalized additive models with interactions.
\newblock In \emph{Machine Learning for Healthcare Conference}, pages 46--79. PMLR, 2020.

\bibitem[Hess(1995)]{hess1995graphical}
Kenneth~R Hess.
\newblock Graphical methods for assessing violations of the proportional hazards assumption in cox regression.
\newblock \emph{Statistics in medicine}, 14\penalty0 (15):\penalty0 1707--1723, 1995.

\bibitem[Ishwaran et~al.(2008)Ishwaran, Kogalur, Blackstone, and Lauer]{ishwaran2008random}
Hemant Ishwaran, Udaya~B Kogalur, Eugene~H Blackstone, and Michael~S Lauer.
\newblock Random survival forests.
\newblock \emph{The annals of applied statistics}, 2\penalty0 (3):\penalty0 841--860, 2008.

\bibitem[Jeanselme et~al.(2022)Jeanselme, Tom, and Barrett]{jeanselme2022neural}
Vincent Jeanselme, Brian Tom, and Jessica Barrett.
\newblock Neural survival clustering: Non-parametric mixture of neural networks for survival clustering.
\newblock In \emph{Conference on Health, Inference, and Learning}, pages 92--102. PMLR, 2022.

\bibitem[Jeanselme et~al.(2023)Jeanselme, Yoon, Tom, and Barrett]{jeanselme2023neural}
Vincent Jeanselme, Chang~Ho Yoon, Brian Tom, and Jessica Barrett.
\newblock Neural fine-gray: Monotonic neural networks for competing risks.
\newblock In \emph{Conference on Health, Inference, and Learning}, pages 379--392. PMLR, 2023.

\bibitem[Jeanselme et~al.(2025)Jeanselme, Tom, and Barrett]{jeanselme2025}
Vincent Jeanselme, Brian Tom, and Jessica Barrett.
\newblock Competing risks: Impact on risk estimation and algorithmic fairness.
\newblock \emph{arXiv preprint arXiv:2508.05435}, 2025.

\bibitem[Jiang(2022)]{jiang2022coxnams}
Zhenjie Jiang.
\newblock Coxnams: Interpretable deep learning model for survival analysis.
\newblock Master's thesis, ETH Zurich, 2022.

\bibitem[Kaplan and Meier(1958)]{kaplan1958nonparametric}
Edward~L Kaplan and Paul Meier.
\newblock Nonparametric estimation from incomplete observations.
\newblock \emph{Journal of the American statistical association}, 53\penalty0 (282):\penalty0 457--481, 1958.

\bibitem[Katzman et~al.(2018)Katzman, Shaham, Cloninger, Bates, Jiang, and Kluger]{katzman2018deepsurv}
Jared~L Katzman, Uri Shaham, Alexander Cloninger, Jonathan Bates, Tingting Jiang, and Yuval Kluger.
\newblock Deepsurv: personalized treatment recommender system using a cox proportional hazards deep neural network.
\newblock \emph{BMC medical research methodology}, 18\penalty0 (1):\penalty0 1--12, 2018.

\bibitem[Ketenci et~al.(2023)Ketenci, Bhave, Elhadad, and Perotte]{ketenci2023maximum}
Mert Ketenci, Shreyas Bhave, Noemie Elhadad, and Adler Perotte.
\newblock Maximum likelihood estimation of flexible survival densities with importance sampling.
\newblock In \emph{Machine Learning for Healthcare Conference}, pages 360--380. PMLR, 2023.

\bibitem[Kingma and Ba(2014)]{kingma2014adam}
Diederik~P Kingma and Jimmy Ba.
\newblock Adam: A method for stochastic optimization.
\newblock \emph{arXiv preprint arXiv:1412.6980}, 2014.

\bibitem[Knaus et~al.(1995)Knaus, Harrell, Lynn, Goldman, Phillips, Connors, Dawson, Fulkerson, Califf, Desbiens, et~al.]{knaus1995support}
William~A Knaus, Frank~E Harrell, Joanne Lynn, Lee Goldman, Russell~S Phillips, Alfred~F Connors, Neal~V Dawson, William~J Fulkerson, Robert~M Califf, Norman Desbiens, et~al.
\newblock The support prognostic model: Objective estimates of survival for seriously ill hospitalized adults.
\newblock \emph{Annals of internal medicine}, 122\penalty0 (3):\penalty0 191--203, 1995.

\bibitem[Kov{\'a}cs(2024)]{kovacs2024feature}
L{\'a}szl{\'o} Kov{\'a}cs.
\newblock Feature selection algorithms in generalized additive models under concurvity.
\newblock \emph{Computational Statistics}, 39\penalty0 (2):\penalty0 461--493, 2024.

\bibitem[Kovalev et~al.(2020)Kovalev, Utkin, and Kasimov]{kovalev2020survlime}
Maxim~S Kovalev, Lev~V Utkin, and Ernest~M Kasimov.
\newblock Survlime: A method for explaining machine learning survival models.
\newblock \emph{Knowledge-Based Systems}, 203:\penalty0 106164, 2020.

\bibitem[Krzyzi{\'n}ski et~al.(2023)Krzyzi{\'n}ski, Spytek, Baniecki, and Biecek]{krzyzinski2023survshap}
Mateusz Krzyzi{\'n}ski, Miko{\l}aj Spytek, Hubert Baniecki, and Przemys{\l}aw Biecek.
\newblock Survshap (t): time-dependent explanations of machine learning survival models.
\newblock \emph{Knowledge-Based Systems}, 262:\penalty0 110234, 2023.

\bibitem[Kvamme(2022)]{havakvpy36:online}
H{\aa}vard Kvamme.
\newblock havakv/pycox: Survival analysis with pytorch.
\newblock \url{https://github.com/havakv/pycox}, 11 2022.

\bibitem[Kvamme et~al.(2019)Kvamme, Borgan, and Scheel]{kvamme2019time}
H{\aa}vard Kvamme, {\O}rnulf Borgan, and Ida Scheel.
\newblock Time-to-event prediction with neural networks and cox regression.
\newblock \emph{Journal of Machine Learning Research}, 2019.

\bibitem[Lee et~al.(2018)Lee, Zame, Yoon, and Van Der~Schaar]{lee2018deephit}
Changhee Lee, William Zame, Jinsung Yoon, and Mihaela Van Der~Schaar.
\newblock Deephit: A deep learning approach to survival analysis with competing risks.
\newblock In \emph{Proceedings of the AAAI conference on artificial intelligence}, volume~32, 2018.

\bibitem[Lee et~al.(2019)Lee, Zame, Alaa, and Schaar]{lee2019temporal}
Changhee Lee, William Zame, Ahmed Alaa, and Mihaela Schaar.
\newblock Temporal quilting for survival analysis.
\newblock In \emph{The 22nd international conference on artificial intelligence and statistics}, pages 596--605. PMLR, 2019.

\bibitem[Li et~al.(2023)Li, Liang, Ma, and Ma]{li2023spatio}
Yang Li, Dongzuo Liang, Shuangge Ma, and Chenjin Ma.
\newblock Spatio-temporally smoothed deep survival neural network.
\newblock \emph{Journal of Biomedical Informatics}, 137:\penalty0 104255, 2023.

\bibitem[Lou et~al.(2013)Lou, Caruana, Gehrke, and Hooker]{lou2013accurate}
Yin Lou, Rich Caruana, Johannes Gehrke, and Giles Hooker.
\newblock {A}ccurate {I}ntelligible {M}odels {w}ith {P}airwise {I}nteractions.
\newblock In \emph{Proceedings of the 19th ACM SIGKDD International Conference on Knowledge Discovery and Data Mining}, pages 623--631, 2013.

\bibitem[Lu et~al.(2023)Lu, Swisher, Chung, Jaffray, and Sidey-Gibbons]{lu2023importance}
Sheng-Chieh Lu, Christine~L Swisher, Caroline Chung, David Jaffray, and Chris Sidey-Gibbons.
\newblock On the importance of interpretable machine learning predictions to inform clinical decision making in oncology.
\newblock \emph{Frontiers in Oncology}, 13:\penalty0 1129380, 2023.

\bibitem[Lundberg and Lee(2017)]{lundberg2017unified}
Scott~M Lundberg and Su-In Lee.
\newblock A unified approach to interpreting model predictions.
\newblock \emph{Advances in neural information processing systems}, 30, 2017.

\bibitem[Mnih et~al.(2016)Mnih, Badia, Mirza, Graves, Lillicrap, Harley, Silver, and Kavukcuoglu]{mnih2016asynchronous}
Volodymyr Mnih, Adria~Puigdomenech Badia, Mehdi Mirza, Alex Graves, Timothy Lillicrap, Tim Harley, David Silver, and Koray Kavukcuoglu.
\newblock Asynchronous methods for deep reinforcement learning.
\newblock In \emph{International conference on machine learning}, pages 1928--1937. PmLR, 2016.

\bibitem[Morita et~al.(2009)Morita, Okamoto, Kobayashi, Yamazaki, Asahina, Inoue, Hagiwara, Sunaga, Yanagitani, Hida, et~al.]{morita2009combined}
Satoshi Morita, Isamu Okamoto, Kunihiko Kobayashi, Koichi Yamazaki, Hajime Asahina, Akira Inoue, Koichi Hagiwara, Noriaki Sunaga, Noriko Yanagitani, Toyoaki Hida, et~al.
\newblock Combined survival analysis of prospective clinical trials of gefitinib for non--small cell lung cancer with egfr mutations.
\newblock \emph{Clinical Cancer Research}, 15\penalty0 (13):\penalty0 4493--4498, 2009.

\bibitem[M{\"u}llner(2011)]{mullner2011modern}
Daniel M{\"u}llner.
\newblock Modern hierarchical, agglomerative clustering algorithms.
\newblock \emph{arXiv preprint arXiv:1109.2378}, 2011.

\bibitem[Nagpal et~al.(2021{\natexlab{a}})Nagpal, Li, and Dubrawski]{nagpal2021dsm}
Chirag Nagpal, Xinyu Li, and Artur Dubrawski.
\newblock Deep survival machines: Fully parametric survival regression and representation learning for censored data with competing risks.
\newblock \emph{IEEE Journal of Biomedical and Health Informatics}, 25\penalty0 (8):\penalty0 3163--3175, 2021{\natexlab{a}}.

\bibitem[Nagpal et~al.(2021{\natexlab{b}})Nagpal, Yadlowsky, Rostamzadeh, and Heller]{nagpal2021dcm}
Chirag Nagpal, Steve Yadlowsky, Negar Rostamzadeh, and Katherine Heller.
\newblock Deep cox mixtures for survival regression.
\newblock In \emph{Machine Learning for Healthcare Conference}, pages 674--708. PMLR, 2021{\natexlab{b}}.

\bibitem[Nagpal et~al.(2022)Nagpal, Goswami, Dufendach, and Dubrawski]{nagpal2022counterfactual}
Chirag Nagpal, Mononito Goswami, Keith Dufendach, and Artur Dubrawski.
\newblock Counterfactual phenotyping with censored time-to-events.
\newblock \emph{Proceedings of the 28th ACM SIGKDD Conference on Knowledge Discovery and Data Mining}, 2022.

\bibitem[Nelson(1969)]{nelson1969hazard}
Wayne Nelson.
\newblock Hazard plotting for incomplete failure data.
\newblock \emph{Journal of Quality Technology}, 1\penalty0 (1):\penalty0 27--52, 1969.

\bibitem[Panahiazar et~al.(2015)Panahiazar, Taslimitehrani, Pereira, and Pathak]{panahiazar2015using}
Maryam Panahiazar, Vahid Taslimitehrani, Naveen Pereira, and Jyotishman Pathak.
\newblock Using ehrs and machine learning for heart failure survival analysis.
\newblock \emph{Studies in health technology and informatics}, 216:\penalty0 40, 2015.

\bibitem[Peroni et~al.(2022)Peroni, Kurban, Yang, Kim, Kang, and Song]{peroni2022extending}
Matthew Peroni, Marharyta Kurban, Sun~Young Yang, Young~Sun Kim, Hae~Yeon Kang, and Ji~Hyun Song.
\newblock Extending the neural additive model for survival analysis with ehr data.
\newblock \emph{arXiv preprint arXiv:2211.07814}, 2022.

\bibitem[Perotte et~al.(2015)Perotte, Ranganath, Hirsch, Blei, and Elhadad]{perotte2015risk}
Adler Perotte, Rajesh Ranganath, Jamie~S Hirsch, David Blei, and No{\'e}mie Elhadad.
\newblock Risk prediction for chronic kidney disease progression using heterogeneous electronic health record data and time series analysis.
\newblock \emph{Journal of the American Medical Informatics Association}, 22\penalty0 (4):\penalty0 872--880, 2015.

\bibitem[Qi et~al.(2023)Qi, Kumar, Farrokh, Sun, Kuan, Ranganath, Henao, and Greiner]{qi2023effective}
Shi-ang Qi, Neeraj Kumar, Mahtab Farrokh, Weijie Sun, Li-Hao Kuan, Rajesh Ranganath, Ricardo Henao, and Russell Greiner.
\newblock An effective meaningful way to evaluate survival models.
\newblock \emph{arXiv preprint arXiv:2306.01196}, 2023.

\bibitem[Ramsay et~al.(2003)Ramsay, Burnett, and Krewski]{ramsay2003effect}
Timothy~O Ramsay, Richard~T Burnett, and Daniel Krewski.
\newblock The effect of concurvity in generalized additive models linking mortality to ambient particulate matter.
\newblock \emph{Epidemiology}, 14\penalty0 (1):\penalty0 18--23, 2003.

\bibitem[Rindt et~al.(2022)Rindt, Hu, Steinsaltz, and Sejdinovic]{rindt2022survival}
David Rindt, Robert Hu, David Steinsaltz, and Dino Sejdinovic.
\newblock Survival regression with proper scoring rules and monotonic neural networks.
\newblock In \emph{International conference on artificial intelligence and statistics}, pages 1190--1205. PMLR, 2022.

\bibitem[Shortliffe and Sep{\'u}lveda(2018)]{shortliffe2018clinical}
Edward~H Shortliffe and Martin~J Sep{\'u}lveda.
\newblock Clinical decision support in the era of artificial intelligence.
\newblock \emph{Jama}, 320\penalty0 (21):\penalty0 2199--2200, 2018.

\bibitem[Siems et~al.(2023)Siems, Ditschuneit, Ripken, Lindborg, Schambach, Otterbach, and Genzel]{siems2023curve}
Julien Siems, Konstantin Ditschuneit, Winfried Ripken, Alma Lindborg, Maximilian Schambach, Johannes Otterbach, and Martin Genzel.
\newblock Curve your enthusiasm: concurvity regularization in differentiable generalized additive models.
\newblock \emph{Advances in Neural Information Processing Systems}, 36:\penalty0 19029--19057, 2023.

\bibitem[Singh and Mukhopadhyay(2011)]{singh2011survival}
Ritesh Singh and Keshab Mukhopadhyay.
\newblock Survival analysis in clinical trials: Basics and must know areas.
\newblock \emph{Perspectives in clinical research}, 2\penalty0 (4):\penalty0 145--148, 2011.

\bibitem[Tan et~al.(2024)Tan, Chen, Zhou, Ong, Sin, Bui, Mehta, Feng, and See]{tan2024association}
Daniel~J Tan, Jiayang Chen, Yirui Zhou, Jaryl Shen~Quan Ong, Richmond Jing~Xuan Sin, Thach~V Bui, Anokhi~Amit Mehta, Mengling Feng, and Kay~Choong See.
\newblock Association of body temperature and mortality in critically ill patients: an observational study using two large databases.
\newblock \emph{European Journal of Medical Research}, 29\penalty0 (1):\penalty0 33, 2024.

\bibitem[Tang et~al.(2023)Tang, Chen, Zhao, Wang, and Tao]{tang2023all}
Liyao Tang, Zhe Chen, Shanshan Zhao, Chaoyue Wang, and Dacheng Tao.
\newblock All points matter: entropy-regularized distribution alignment for weakly-supervised 3d segmentation.
\newblock \emph{Advances in Neural Information Processing Systems}, 36:\penalty0 78657--78673, 2023.

\bibitem[Tonekaboni et~al.(2019)Tonekaboni, Joshi, McCradden, and Goldenberg]{tonekaboni2019clinicians}
Sana Tonekaboni, Shalmali Joshi, Melissa~D McCradden, and Anna Goldenberg.
\newblock What clinicians want: contextualizing explainable machine learning for clinical end use.
\newblock In \emph{Machine learning for healthcare conference}, pages 359--380. PMLR, 2019.

\bibitem[Uno et~al.(2011)Uno, Cai, Pencina, D'Agostino, and Wei]{uno2011c}
Hajime Uno, Tianxi Cai, Michael~J Pencina, Ralph~B D'Agostino, and Lee-Jen Wei.
\newblock On the c-statistics for evaluating overall adequacy of risk prediction procedures with censored survival data.
\newblock \emph{Statistics in medicine}, 30\penalty0 (10):\penalty0 1105--1117, 2011.

\bibitem[Utkin et~al.(2022)Utkin, Satyukov, and Konstantinov]{utkin2022survnam}
Lev~V Utkin, Egor~D Satyukov, and Andrei~V Konstantinov.
\newblock Survnam: The machine learning survival model explanation.
\newblock \emph{Neural Networks}, 147:\penalty0 81--102, 2022.

\bibitem[Vigan{\`o} et~al.(2000)Vigan{\`o}, Dorgan, Buckingham, Bruera, and Suarez-Almazor]{vigano2000survival}
Antonio Vigan{\`o}, Marlene Dorgan, Jeanette Buckingham, Eduardo Bruera, and Maria~E Suarez-Almazor.
\newblock Survival prediction in terminal cancer patients: a systematic review of the medical literature.
\newblock \emph{Palliative Medicine}, 14\penalty0 (5):\penalty0 363--374, 2000.

\bibitem[Virtanen et~al.(2020)Virtanen, Gommers, Oliphant, Haberland, Reddy, Cournapeau, Burovski, Peterson, Weckesser, Bright, et~al.]{virtanen2020scipy}
Pauli Virtanen, Ralf Gommers, Travis~E Oliphant, Matt Haberland, Tyler Reddy, David Cournapeau, Evgeni Burovski, Pearu Peterson, Warren Weckesser, Jonathan Bright, et~al.
\newblock Scipy 1.0: fundamental algorithms for scientific computing in python.
\newblock \emph{Nature methods}, 17\penalty0 (3):\penalty0 261--272, 2020.

\bibitem[Wang and Sun(2022)]{wang2022survtrace}
Zifeng Wang and Jimeng Sun.
\newblock Survtrace: Transformers for survival analysis with competing events.
\newblock In \emph{Proceedings of the 13th ACM International Conference on Bioinformatics, Computational Biology and Health Informatics}, pages 1--9, 2022.

\bibitem[Xu and Guo(2023)]{xu2023coxnam}
Liangchen Xu and Chonghui Guo.
\newblock Coxnam: An interpretable deep survival analysis model.
\newblock \emph{Expert Systems with Applications}, page 120218, 2023.

\end{thebibliography}

\newpage
\appendix
\section{Notation Table}
\label{app:notation}

\begin{table}[H]
\begin{tabular}{ll}
\textbf{Symbol} & \textbf{Meaning}                              \\
$N$             & Dataset size                                  \\
$D$             & Feature dimensionality                        \\
$C$             & Subgroup size                  \\
$\mathcal{D}$   & Empirical dataset                             \\

$\tilde{\Tb}_{LM}$ & Importance time samples to approximate loglikelihood objective\\

$i$             & Data instance index                           \\
$d$             & Feature index                                 \\
$h$             & Correlation threshold to merge subgroups \\
$\xb$            & Observed covariates                           \\
$\delta$         & Censoring index                               \\
$t$             & Recorded time-to-ecent or censoring time      \\
$\phi_d$        & Population-level hazard network parameter     \\
$\theta$        & Subgroup assignment network parameter               \\
$\Phi$          & $\Phi=\{\phi_d\}_{d=1}^D$                    
\end{tabular}\caption{Notation Table}
\end{table}

\section{Flow chart}
\label{app:flow}

 \begin{figure}[H]
\includegraphics[width=\linewidth]{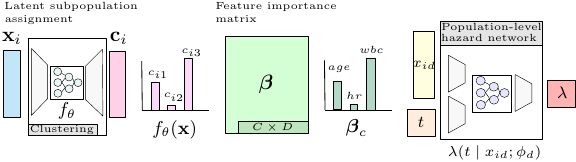}
\caption{Flowchart of ADHAM. The covariates are mapped to latent subgroups, which are used to index the rows of $\beta$. The population-level hazard curves are calculated via $\lambda(t \mid x_{id}; \phi_d)$, and the marginal hazard is calculated by marginalizing the population-level curve using subgroup-level weights, and instance-specific subgroup assignments.
}\label{fig:motivation}
\end{figure}

\section{Model Selection}

\subsection{Merging Same subgroups, with $\boldsymbol{\beta}_{c_1} = \boldsymbol{\beta}_{c_2}$, Retains the Data Generating Log-likelihood and Predictive Performance}\label{app:proof}

In this section, we show that our model selection procedure, outlined in Algorithm \ref{alg:refine}, does not change the data log-likelihood and predictive performance if $\boldsymbol{\beta}_{c_1} = \boldsymbol{\beta}_{c_2}$. 

\begin{prop}
if $\boldsymbol{\beta}_{c_1} = \boldsymbol{\beta}_{c_2}$, for $c_1$ and $c_2 \in \mathcal{C} = \{1,2, \cdots, C\}$, then grouping $c_1$ and $c_2$ into a new group $c^\star$, where $\cp{c^\star}{\xb; \theta} = \cp{c_1}{\xb; \theta}  + \cp{c_2}{\xb; \theta}$ does not change the data generating log-likelihood and risk predictions. 
\end{prop}

\begin{proof} Consider an input patient $\xb$ and a model in which two groups, $c_1$ and $c_2$, share identical covariate importance vectors, i.e., $\boldsymbol{\beta}_{c_1} = \boldsymbol{\beta}_{c_2}$. To complete the proof, it suffices to show that their combined contribution to the marginal hazard function remains unchanged when merged into a single group $c^\star$, with assignment probability $\cp{c^\star}{\xb; \theta} = \cp{c_1}{\xb; \theta} + \cp{c_2}{\xb; \theta}$:

\begin{align}
    &\sum_{c=1}^C \sum_{d=1}^D  \beta_{dc}\cp{c}{\xb; \theta}\lambda(t \mid x_{id}; \phi_d) \nonumber \\
    &= \sum_{\substack{c\neq c_1 \\ c \neq c_2}} \sum_{d=1}^D  \beta_{dc}\cp{c}{\xb; \theta}\lambda(t \mid x_{id}; \phi_d) + \sum_{c\in \{ c_1, c_2\}} \sum_{d=1}^D  \beta_{dc}\cp{c}{\xb; \theta}\lambda(t \mid x_{id}; \phi_d) \nonumber \\
    &=  \sum_{\substack{c\neq c_1 \\ c \neq c_2}} \sum_{d=1}^D  \beta_{dc}\cp{c}{\xb; \theta}\lambda(t \mid x_{id}; \phi_d) + \left(\cp{c_1}{\xb; \theta}  + \cp{c_2}{\xb; \theta}\right) \sum_{d=1}^D \beta_{dc^\star} \lambda(t \mid x_{id}; \phi_d) \nonumber\\
    &= \sum_{\substack{c\neq c_1 \\ c \neq c_2}} \sum_{d=1}^D  \beta_{dc}\cp{c}{\xb; \theta}\lambda(t \mid x_{id}; \phi_d) + \cp{c^\star}{\xb; \theta} \sum_{d=1}^D \beta_{dc^\star} \lambda(t \mid x_{id}; \phi_d)\nonumber\\
    &= \sum_{c \in  \mathcal{C}^\star} \sum_{d=1}^D \cp{c}{\xb; \theta}\beta_{dc^\star} \lambda(t \mid x_{id}; \phi_d).
\end{align}

The resulting set is updated to $\mathcal{C}^\star = c^\star \cup \mathcal{C} \setminus \{c_1, c_2\}$, where $\mathcal{C} := \{1, 2, \dots, C\}$. Since the marginal hazard function is the same, the model data log-likelihood and predictive performance do not change.
\end{proof}
In practice, it is uncommon for $\boldsymbol{\beta}_{c_1}$ and $\boldsymbol{\beta}_{c_2}$ to be exactly identical. However, as outlined in Algorithm~\ref{alg:refine}, we can define a similarity measure and apply a threshold to merge subgroups, using a performance metric to guide this process—as demonstrated in the following section.

\newpage

\section{Multi-level Interpretability of Survival Function}\label{app:multilevel}

Population-level survival functions can be computed by:

\begin{align}
     S(t \mid x_d; \phi_d)
    &= \expp{- \int_0^t \lambda(s\mid x_d; \phi_d)\diff{s}}\label{eq:poplevels}
\end{align}

This function outputs a probability value, between 0 and 1, and is the same across the population for the same $x_d$ value.

Individual-level survival function is the composition of population-level survival functions exponentiated by patient-specific weights, $\cp{d}{\xb; \theta, \boldsymbol{\beta}}$:

\begin{align}
     S(t \mid \xb; \theta, \boldsymbol{\beta}, \Phi)
    &= \expp{- \sum_{d=1}^D \left(\sum_{c=1}^C \cp{d}{c; \boldsymbol{\beta}}\cp{c}{\xb;\theta}\lambda(t \mid x_d; \phi_d) \diff{s} \right)}\\
    &= \prod_{c=1}^C \left(\prod_{d=1}^D \left({\expp{-\lambda(s\mid x_d; \phi_d)\diff{s}}}\right)^{{\beta_{dc}}} \right)^{{f_{\theta_c}(\xb)}}\\
    &= \prod_{d=1}^D \expp{-\lambda(s\mid x_d; \phi_d)\diff{s}}^{\cp{d}{\xb; \theta, \boldsymbol{\beta}}}
\end{align}

Note that, $ {\expp{-\lambda(s\mid x_d; \phi_d)\diff{s}}}$ is the same across the population for the same $x_d$. ${\beta_{dc}}$ adjusts the latter for subgroup (\ie for a given subgroup $c$ the survival function differs by some exponent ${\beta_{dc}}$). Finally, ${f_{\theta_c}(\xb)}$ modulates the subgroup weights given a patient covariates. 

\newpage

\section{Population Level Curves by Different Runs}\label{app:diffruns}

In this section, we compare the population-level survival functions of TimeNAM (\textbf{top}) and ADHAM (\textbf{bottom}) on \textbf{SUPPORT} dataset over different runs (each with identical random seeds).\footnote{While it is possible to compare individual curves, we focus on population-level interpretability over different model runs as it provides a higher degree of information across the entire dataset rather than select patients.} Higher values imply longer time-to-event (\ie lower risk of observing the event within time $t$). Survival functions for TimeNAM and ADHAM both capture the well-established trend of increasing risk with age (\ie survival probability decreases with age). ADHAM also captures the well-known physiological trends in heart rate and temperature. In particular, ADHAM highlights the values linked to normal ranges (\eg 36 - 37.5 $^{\circ}$C for temperature and 60 - 100 bpm for heart rate~\citep{tan2024association}), where risk values are lower, while it is harder to observe this for TimeNAM. \emph{The results are consistent for different runs.}

\paragraph{Run 1.} $\;$

\begin{figure}[H]
\centering
\includegraphics[width=\linewidth]{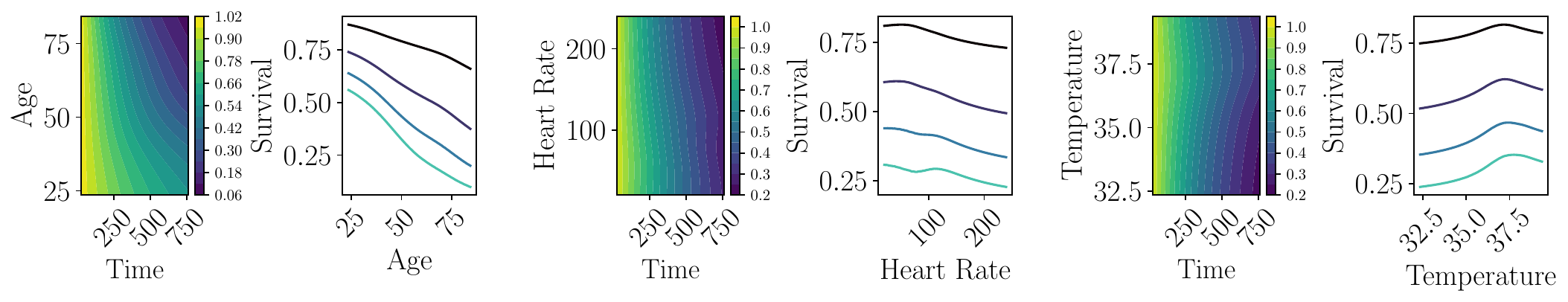}
\end{figure}
\begin{figure}[H]
\centering
\includegraphics[width=\linewidth]{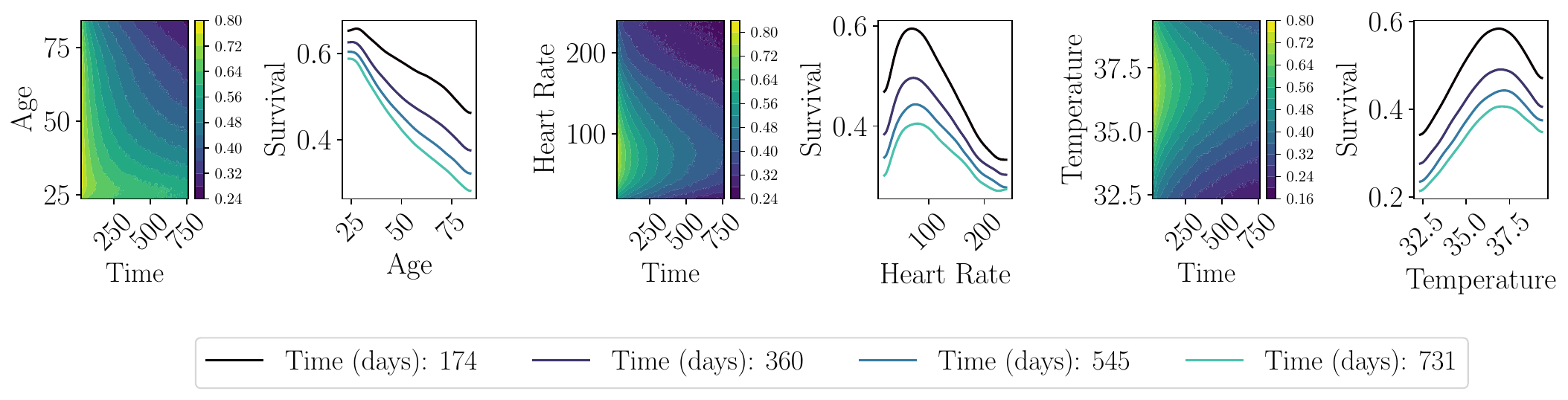}
\caption{Covariate-specific \emph{population-level} risk functions of TimeNAM (\textbf{top}) and ADHAM (\textbf{bottom}) trained on \textbf{SUPPORT} dataset on run 1.}
\end{figure}

\newpage

\paragraph{Run 2.} $\;$

\begin{figure}[H]
\centering
\includegraphics[width=\linewidth]{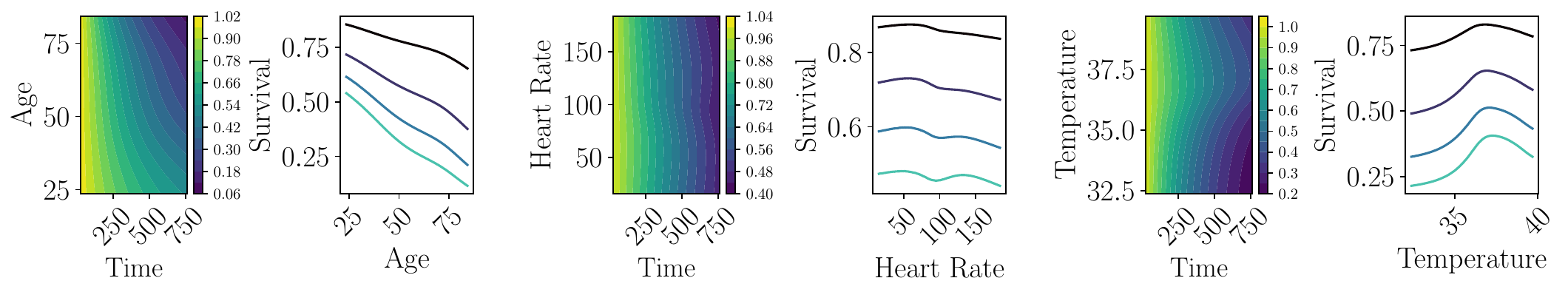}
\end{figure}
\begin{figure}[H]
\centering
\includegraphics[width=\linewidth]{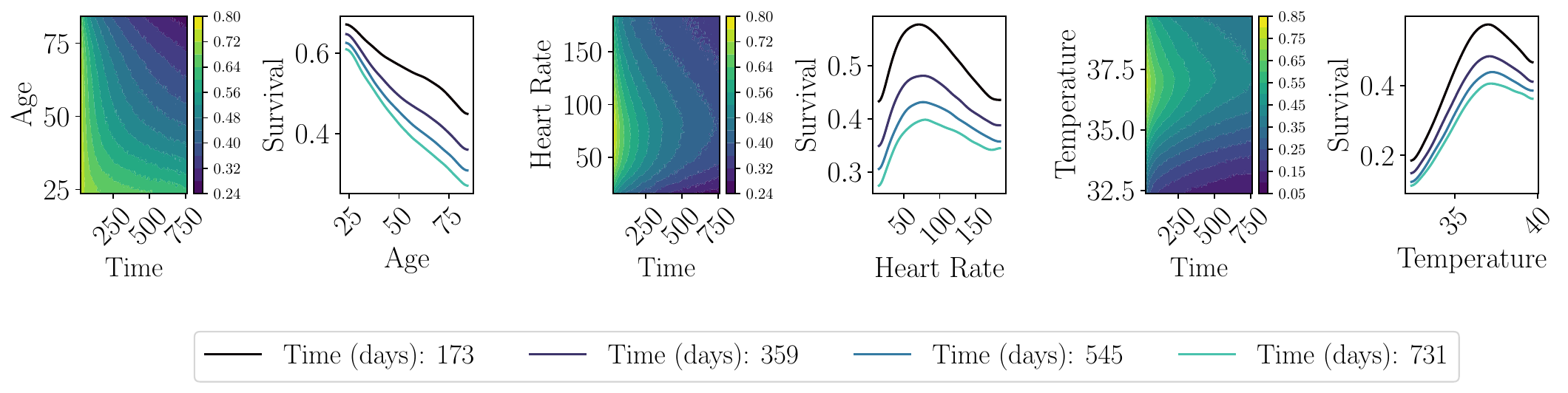}
\caption{Covariate-specific \emph{population-level} risk functions of TimeNAM (\textbf{top}) and ADHAM (\textbf{bottom}) trained on \textbf{SUPPORT} dataset on run 2.}
\end{figure}

\paragraph{Run 3.}$\;$

\begin{figure}[h!]
\centering
\includegraphics[width=\linewidth]{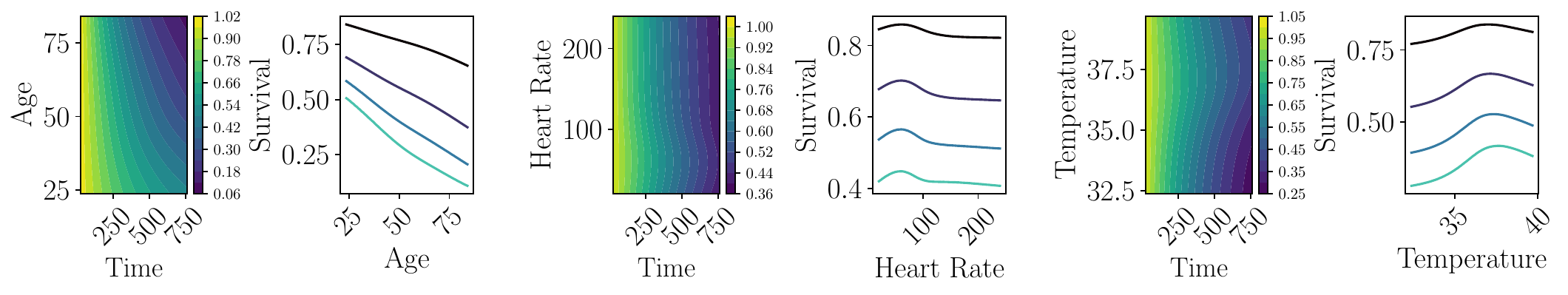}
\end{figure}
\begin{figure}[h!]
\centering
\includegraphics[width=\linewidth]{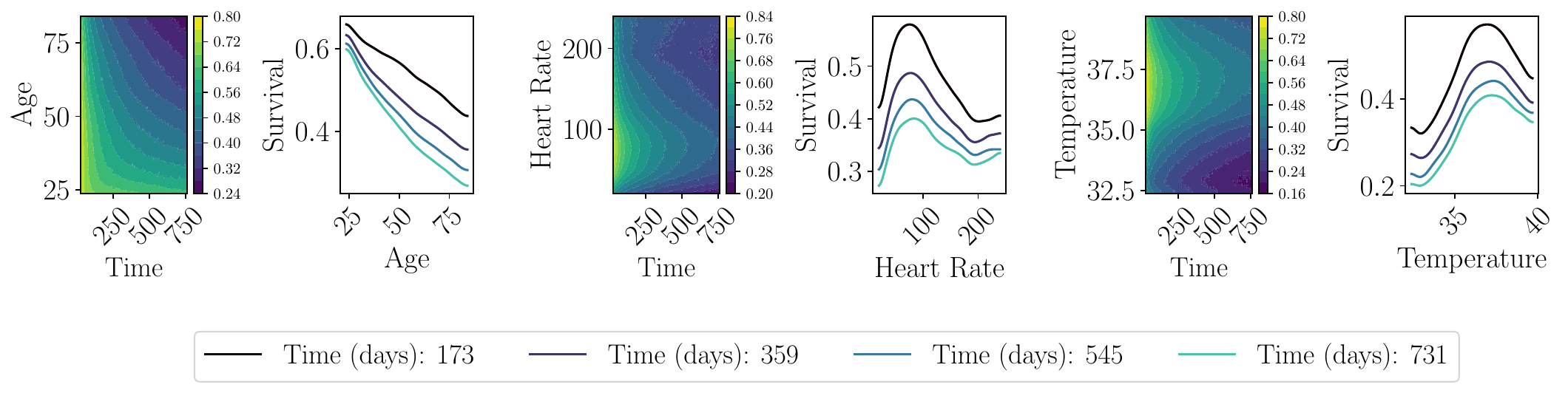}
\caption{Covariate-specific \emph{population-level} risk functions of TimeNAM (\textbf{top}) and ADHAM (\textbf{bottom}) trained on \textbf{SUPPORT} dataset on run 3.}
\end{figure}

$\;$

$\;$

$\;$

\paragraph{Run 4.}$\;$

\begin{figure}[h!]
\centering
\includegraphics[width=\linewidth]{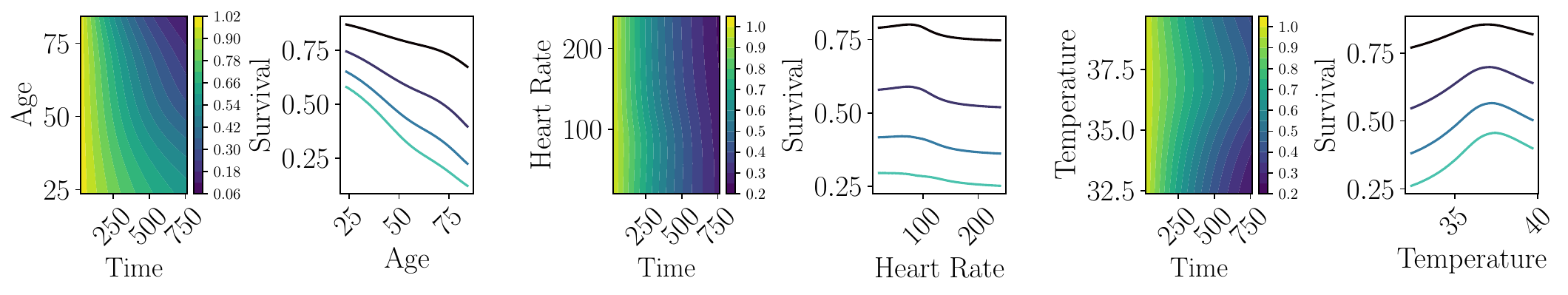}
\end{figure}
\begin{figure}[h!]
\centering
\includegraphics[width=\linewidth]{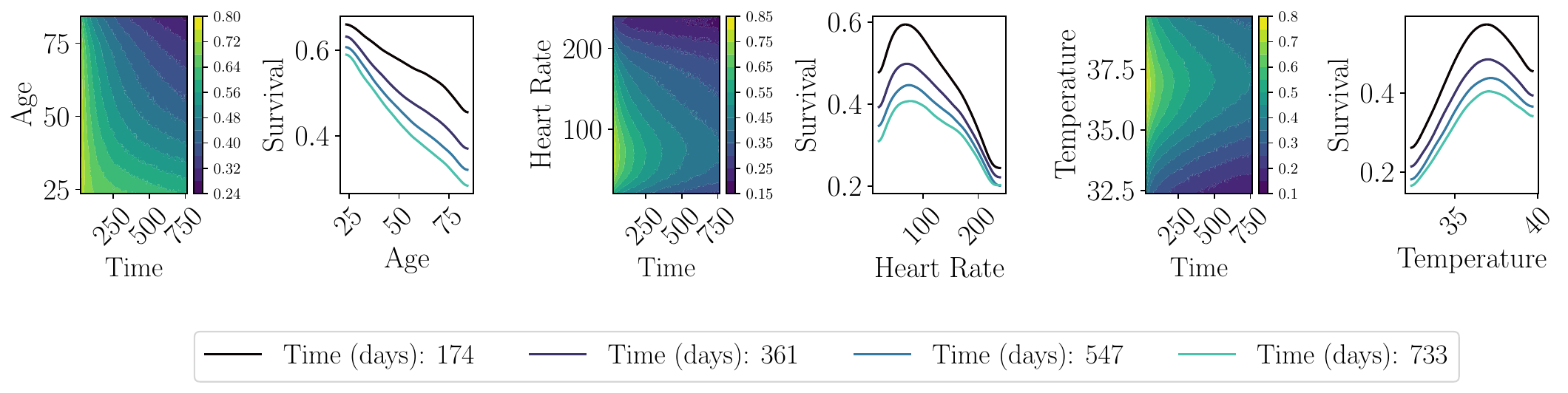}
\caption{Covariate-specific \emph{population-level} risk functions of TimeNAM (\textbf{top}) and ADHAM (\textbf{bottom}) trained on \textbf{SUPPORT} dataset on run 4.}
\end{figure}

\paragraph{Run 5.}$\;$

\begin{figure}[h!]
\centering
\includegraphics[width=\linewidth]{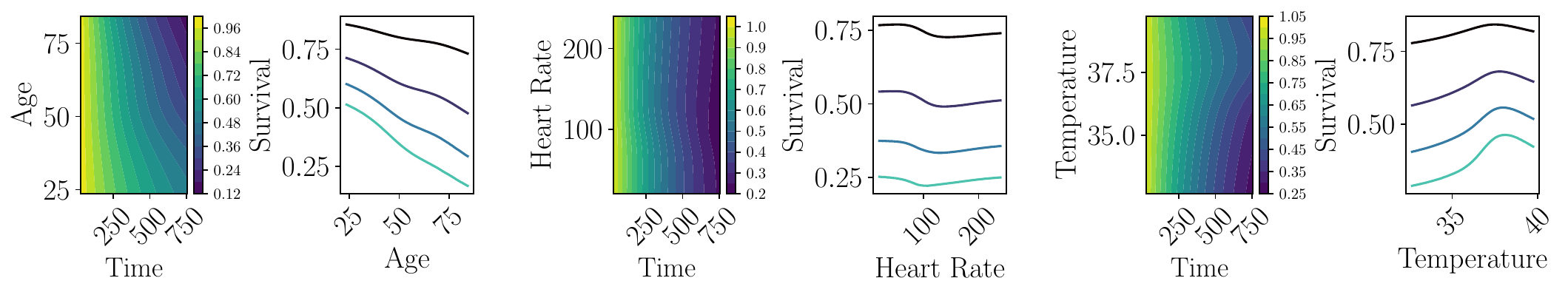}
\end{figure}
\begin{figure}[h!]
\centering
\includegraphics[width=\linewidth]{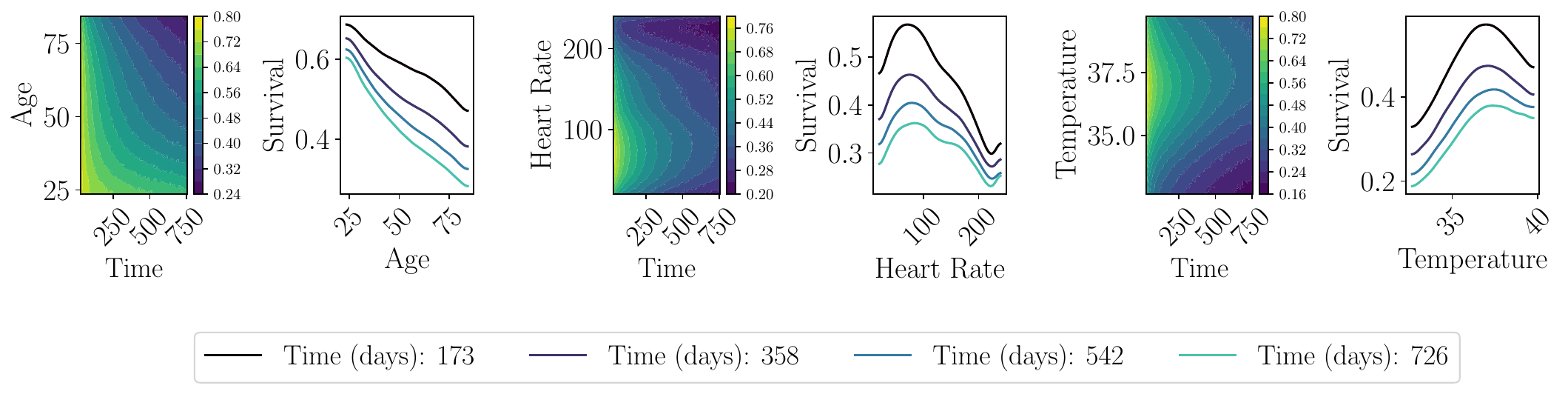}
\caption{Covariate-specific \emph{population-level} risk functions of TimeNAM (\textbf{top}) and ADHAM (\textbf{bottom}) trained on \textbf{SUPPORT} dataset on run 5.}
\end{figure}

\newpage

\section{Experimental Setup Details}\label{app:experimental}

In this section, we explain our datasets in detail.

\subsection{Dataset Details}\label{app:dataset}

\paragraph{SUPPORT.} The Study to Understand Prognoses Preferences Outcomes and Risks of Treatment dataset~\citep{knaus1995support}. After preprocessing with the~\texttt{PyCox} library~\citep{havakvpy36:online}, there are 8,873 patients and 23 covariates with a median follow-up of 231 days and a censoring rate of 31.9\%.


\paragraph{FLCHAIN.} Data collected from a controlled trial in Olmsted County, Minnesota, comprised of assays of serum free light chain (FLCHAIN) and mortality data~\citep{dispenzieri2012use}. There are 6524 patients with 16 covariates with a median follow-up of 4,303 days and a censoring rate of 70\% after preprocessing with~\texttt{PyCox}.

\paragraph{CKD.} Electronic health record data from a large urban hospital is used for this dataset. It comprises a cohort of patients with incident chronic kidney disease (CKD) where the event of interest is in-hospital diagnosis of acute kidney injury. The dataset contains 10,173 patients and 33 covariates. The median follow-up days and censoring rate are 67 days and 64\%, respectively. This dataset was used in~\citep{ketenci2023maximum}, one of our comparator models.

\subsection{SUPPORT Dataset Abbreviations}\label{app:covariates}

Here, we explain the abbreviations used in the \textbf{SUPPORT} dataset for Figures \ref{fig:motivation}, \ref{fig:clusters}, and \ref{fig:hazard}.

\begin{figure}[!htb] \raggedright
\begin{flushleft}\raggedright\hspace{-10ex}
 \includegraphics[width=0.9\linewidth]{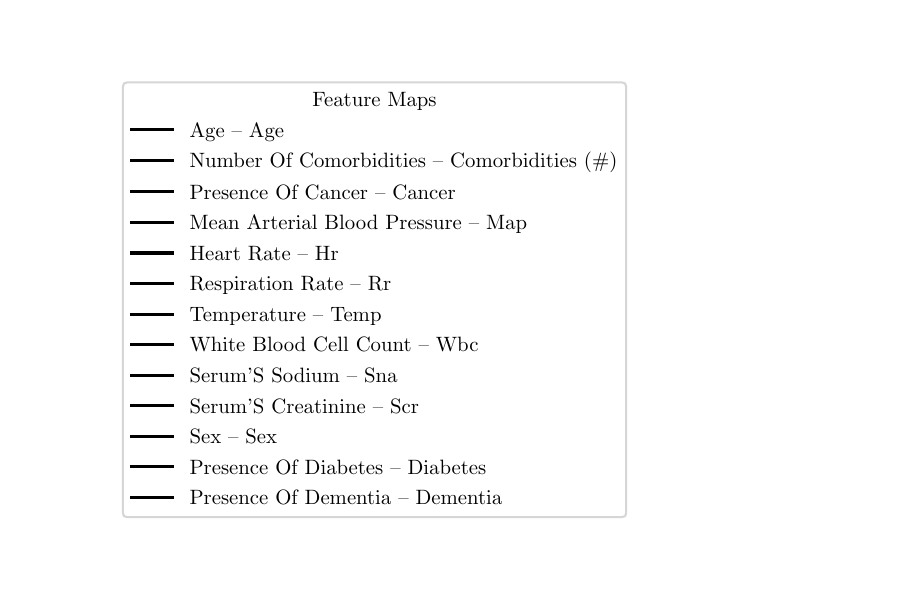}
\end{flushleft}
\end{figure}

\vspace{-5ex}

\subsection{Baseline Model Details} \label{app:baselines}

\paragraph{CoxPH.} The semi-parametric Cox proportional hazards model~\citep{cox1972regression}. Parameter learning is carried out by optimizing the partial log-likelihood.

\paragraph{DeepSurv.} A semi-parametric survival model that improves the Cox Proportional Hazards model by employing neural networks, thereby establishing a non-linear proportional hazards function~\citep{katzman2018deepsurv}.

\paragraph{Random Survival Forests (RSF).} An extension of random forests that fits multiple trees to survival data by bagging and using the cumulative hazard function computed by the Nelson-Aalen estimator ~\citep{ishwaran2008random}. 

\paragraph{DeepHit.} A discrete-time survival model parameterized by neural networks with a softmax output layer. DeepHit uses cross-entropy loss combined with a ranking loss~\citep{lee2018deephit}. 

\paragraph{Cox-Time.} A semi-parametric method that extends Cox analysis beyond proportional hazards. Cox-Time uses neural networks to parameterize the hazard function. Parameter estimation is done by optimizing a biased approximation of the partial log-likelihood ~\citep{kvamme2019time}.

\paragraph{TimeNAM \& TimeNA2M.} TimeNA2M extends Cox-Time by using a neural additive hazard function that explicitly captures both main effects and pairwise (second-order) interactions between covariates~\citep{agarwal2021neural, peroni2022extending}.

\paragraph{Deep Survival Machines (DSM).} A parametric survival model that uses a mixture of Weibull and log-normal distributions. Mixture assignments and time-to-event distributions are parameterized by neural networks conditioned on covariates. Parameter estimation is done by optimizing the ELBO, where the expectation is taken with respect to the conditional model prior ~\citep{nagpal2021dsm}.

\paragraph{Deep Cox Mixtures (DCM).} An semi-parametric extension of DSM where each time-to-event density is modeled by DeepSurv. DCM assumes that the hazards within subgroups is proportional. Parameter estimation is done by the Expectation-Maximization (EM) algorithm and fitting polynomial splines to baseline hazards ~\citep{nagpal2021dcm}. 

\paragraph{Deep Hazard Analysis (DHA).} A fully parametric survival analysis approach that directly models the hazard function. The intractable cumulative hazard integral is handled by importance sampling and learning is performed by exact log-likelihood optimization ~\citep{ketenci2023maximum}.

\subsection{Hyperparameter Details}\label{app:hyper}

All models have an equal training length of 4000 epochs. We pick the best-performing model with respect to their validation loss. The hyper-parameter spaces of each benchmark model are listed below.

\paragraph{CoxPH.}

    \begin{itemize}
        \item[] 
            \textquoteleft alpha': [0, 1e-3, 1e-2, 1e-1],
    \end{itemize}

\paragraph{DeepSurv.}

    \begin{itemize}
        \item[] 
        
            \textquoteleft lr' : [5e-4, 1e-3],

            \textquoteleft batch\_size': [256, 512, 1024],

            \textquoteleft weight\_decay': [0, 1e-8, 1e-6, 1e-3, 1e-1],

            \textquoteleft nodes\_':[128, 256, 512],

            \textquoteleft layers\_': [2, 3],

            \textquoteleft dropout': [0, 1e-1, 2e-1, 4e-1, 5e-1],

    \end{itemize}

\paragraph{RSF.}

    \begin{itemize}
        \item[] 
                \textquoteleft max\_depth' : [None, 5],

                \textquoteleft n\_estimators' : [50, 100, 150, 200, 150],
                
                \textquoteleft max\_covariates' : [50, 75, sqrt(d), d//2, d],

                \textquoteleft min\_samples\_split' : [10, 150, 200, 250],
                
    \end{itemize}

\textquoteleft max\_depth':None means that the expansion continues until all leaves are pure.

\paragraph{DSM.}

    \begin{itemize}
        \item[] 
        
         \textquoteleft k \textquoteright : [3, 4, 6],
         
         \textquoteleft distribution' : [\textquoteleft Weibull', \textquoteleft LogNormal'],
        
        \textquoteleft learning\_rate' : [1e-4, 5e-4, 1e-3],

        \textquoteleft nodes\_' : [48, 64, 96, 256],

        \textquoteleft hidden\_layers\_': [1, 2, 3],
        
        \textquoteleft discount': [1/3, 3/4, 1],
        
        \textquoteleft batch\_size': [128, 256],

    \end{itemize}

\paragraph{DCM.}

    \begin{itemize}
        \item[] 
        
            \textquoteleft k' : [3, 4, 6],
            
            \textquoteleft nodes\_' : [48, 64, 96, 256],
            
            \textquoteleft hidden\_layers\_': [1, 2, 3],
            
            \textquoteleft batch\_size': [128, 256],
            
            \textquoteleft use\_activation': [True, False],

    \end{itemize}
    
\paragraph{Deep-Hit.}

    \begin{itemize}
        \item[] 
        
            \textquoteleft lr' : [5e-4, 1e-3],

            \textquoteleft batch\_size': [256, 512, 1024],
            
            \textquoteleft weight\_decay': [0, 1e-8, 1e-6, 1e-3, 1e-1],

            \textquoteleft nodes\_':[128, 256, 512],

            \textquoteleft hidden\_layers\_': [2, 3],

            \textquoteleft dropout': [0, 1e-1, 2e-1, 4e-1, 5e-1],

            \textquoteleft alpha': [1e-1, 2e-1, 4e-1, 8e-1, 1],
  
            \textquoteleft sigma' : [1e-1, 2.5e-1, 4e-1, 8e-1, 1, 2, 10],

            \textquoteleft num\_durations' : [10, 50, 100],

    \end{itemize}

\paragraph{Cox-Time.}

    \begin{itemize}
        \item[] 
        
            \textquoteleft lr' : [5e-4, 1e-3],

            \textquoteleft batch\_size': [256, 512, 1024],
            
            \textquoteleft weight\_decay': [0, 1e-8, 1e-6, 1e-3, 1e-1],

            \textquoteleft nodes\_':[128, 256, 512],

            \textquoteleft hidden\_layers\_': [1,2,3],

            \textquoteleft dropout': [0, 1e-1, 2e-1, 4e-1, 5e-1],

            \textquoteleft lambda': [0, 1e-3, 1e-2, 1e-1],
  
            \textquoteleft log\_duration' : [True, False],

    \end{itemize}

\paragraph{TimeNAM \& TimeNA2M.}

    \begin{itemize}
        \item[] 
        
            \textquoteleft lr' : [5e-4, 1e-3],

            \textquoteleft batch\_size': [256, 512, 1024],
            
            \textquoteleft weight\_decay': [0, 1e-8, 1e-6, 1e-3, 1e-1],

            \textquoteleft nodes\_':[128, 256, 512],

            \textquoteleft hidden\_layers\_': [1,2,3],

            \textquoteleft dropout': [0, 1e-1, 2e-1, 4e-1, 5e-1],

            \textquoteleft lambda': [0, 1e-3, 1e-2, 1e-1],
  
            \textquoteleft log\_duration' : [True, False],

    \end{itemize}

\paragraph{DHA.}

    \begin{itemize}
        \item[] 
        
            \textquoteleft lr' : 1e-3,

            \textquoteleft batch\_size': 512,

            \textquoteleft imps\_size': 64,

            \textquoteleft layer\_norm' : True,
            
            \textquoteleft weight\_decay':0,

            \textquoteleft nodes\_': 100

            \textquoteleft layers\_': 3,

            \textquoteleft dropout': 0.15,

            \textquoteleft act': elu,
    \end{itemize}

\paragraph{ADHAM.}

    \begin{itemize}
        \item[] 
        
            \textquoteleft lr' : 1e-3,

            \textquoteleft batch\_size': 512,

            \textquoteleft imps\_size': 64,

            \textquoteleft layer\_norm' : True,
            
            \textquoteleft weight\_decay':0,

            \textquoteleft nodes\_': 100

            \textquoteleft layers\_': 3,

            \textquoteleft dropout': [0, 0.15],

            \textquoteleft act': elu,

            \textquoteleft n\_mixtures': 100,

            \textquoteleft add\_const': [True, False],

    \end{itemize}

For DHA and ADHAM, we use the architecture A1 defined in~\citep{ketenci2023maximum}.

\newpage

\subsection{Evaluation Metric Details}\label{app:metric}

In this section, we explain the evaluation metrics.

\paragraph{Concordance Index (C-Index).} The C-Index measures how well a model ranks individuals with respect to their risk. We use the inverse probability of censoring weighting (IPCW)-based C-Index proposed by~\cite{uno2011c}, which evaluates the agreement between predicted and observed orderings of events, with a time cutoff $t$. It is defined as:

\begin{equation}
  \mathrm{C}(t) =  \cp { S(t | \xb_i) < S(t|\xb_j) }{ \delta_i = 1, t _i < t _j, t _i<t }.
\end{equation}

Larger C-Index is better.




\paragraph{Brier Score (BS).} The Brier Score~\citep{brier1950verification} measures the calibration of predicted survival probabilities. It is defined as the expected squared difference between the predicted survival probability and the event occurrence indicator. To handle censored data, we use the extension of the Brier Score proposed by~\citep{graf1999assessment}, which employs inverse probability of censoring weighting (IPCW)\footnote{Although, IPCW can fail to provide accurate estimates for censored individuals when there are no comparable individuals who experienced the event afterward, it is a widely used method to correct for censoring~\citep{qi2023effective}.}. The definition is:

\begin{equation}
   \mathrm{BS}(t) = \Ex{t_i, \xb_i \sim \mathcal{D}}{(I_{t_i>t} - S(t|\xb_i))^2}.
\end{equation}

Smaller BS is better.




\paragraph{Area Under the Receiver Operating Characteristic (AUROC).} The time-dependent AUROC at a fixed cutoff time  measures the model's ability to distinguish between individuals who experience the event by time $t$ and those who do not. It is defined as the probability that the predicted risk (i.e., one minus the survival probability) is higher for an individual who experiences the event before or at time $t$ than for one who survives beyond $t$:

\begin{equation}
   \mathrm{AUC}(t) = \cp{S(t|\xb_i) \leq  S(t|\xb_j)}{ t_i \leq t, t_j > t}.
\end{equation}

Larger AUROC is better.

\newpage

\section{The Standard Error of the Sample Mean}\label{app:err}

In this section, we demonstrate the standard error results.

\begin{table}[h]
\center
\begin{adjustbox}{width=0.9\textwidth}
\begin{tabular}{cccccccccc}
\hline
& \multicolumn{9}{c}{\textbf{SUPPORT} Dataset}\\ \hline
\multicolumn{1}{c|}{} & \multicolumn{3}{c|}{25$^\text{th}$ Quantile} & \multicolumn{3}{c|}{50$^\text{th}$ Quantile} & \multicolumn{3}{c}{75$^\text{th}$ Quantile} \\ \cline{2-10} 
\multicolumn{1}{c|}{\multirow{-2}{*}{Models}}& C-Index $\uparrow$ & BS $\downarrow$  & \multicolumn{1}{c|}{AUROC $\uparrow$} & C-Index $\uparrow$ & BS $\downarrow$  & \multicolumn{1}{c|}{AUROC $\uparrow$}  & C-Index $\uparrow$  & BS $\downarrow$    & AUROC $\uparrow$ \\
\hline
\multicolumn{1}{c|}{DeepSurv}& 0.008 & 0.002 & 0.009 & 0.004 & 0.002 & 0.005 & 0.004 & 0.001 & 0.005 \\
\multicolumn{1}{c|}{RSF}& 0.005 & 0.003 & 0.005 & 0.003 & 0.001 & 0.004 & 0.002 & 0.002 & 0.004 \\
\multicolumn{1}{c|}{DeepHit}& 0.009 & 0.003 & 0.008 & 0.005 & 0.003 & 0.006 & 0.005 & 0.002 & 0.006 \\
\multicolumn{1}{c|}{Cox-Time}& 0.013 & 0.002 & 0.012 & 0.007 & 0.002 & 0.008 & 0.004 & 0.001 & 0.004 \\
\multicolumn{1}{c|}{DSM}& 0.006 & 0.002 & 0.004 & 0.005 & 0.002 & 0.006 & 0.004 & 0.002 & 0.005 \\
\multicolumn{1}{c|}{DCM}& 0.009 & 0.003 & 0.008 & 0.006 & 0.002 & 0.006 & 0.005 & 0.002 & 0.005 \\
\multicolumn{1}{c|}{DHA}& 0.008 & 0.003 & 0.007 & 0.005 & 0.002 & 0.006 & 0.005 & 0.002 & 0.006 \\\hline
\multicolumn{1}{c|}{CoxPH}& 0.007 & 0.002 & 0.007 & 0.003 & 0.002 & 0.003 & 0.002 & 0.001 & 0.004 \\
\multicolumn{1}{c|}{TIMENAM}& 0.008 & 0.002 & 0.008 & 0.003 & 0.002 & 0.003 & 0.002 & 0.001 & 0.003 \\
\multicolumn{1}{c|}{TIMENA2M}& 0.009 & 0.003 & 0.009 & 0.004 & 0.002 & 0.004 & 0.004 & 0.001 & 0.003 \\
\multicolumn{1}{c|}{ADHAM ($\mathcal{R}$)}& 0.006 & 0.002 & 0.004 & 0.004 & 0.002 & 0.006 & 0.003 & 0.000 & 0.003 \\
\multicolumn{1}{c|}{ADHAM}& 0.008 & 0.003 & 0.005 & 0.004 & 0.001 & 0.006 & 0.003 & 0.001 & 0.004 \\
\hline     
\end{tabular}
\end{adjustbox}
\vspace{-1ex}
\caption{Standard error of the sample mean (SEM) results on the \textbf{SUPPORT} dataset.}
\label{tab:superr}
\end{table}
\begin{table}[h]
\center
\begin{adjustbox}{width=0.9\textwidth}
\begin{tabular}{cccccccccc}
\hline
& \multicolumn{9}{c}{\textbf{CKD} Dataset}\\ \hline
\multicolumn{1}{c|}{} & \multicolumn{3}{c|}{25$^\text{th}$ Quantile} & \multicolumn{3}{c|}{50$^\text{th}$ Quantile} & \multicolumn{3}{c}{75$^\text{th}$ Quantile} \\ \cline{2-10} 
\multicolumn{1}{c|}{\multirow{-2}{*}{Models}}& C-Index $\uparrow$ & BS $\downarrow$  & \multicolumn{1}{c|}{AUROC $\uparrow$} & C-Index $\uparrow$ & BS $\downarrow$  & \multicolumn{1}{c|}{AUROC $\uparrow$}  & C-Index $\uparrow$  & BS $\downarrow$    & AUROC $\uparrow$ \\
\hline
\multicolumn{1}{c|}{DeepSurv}& 0.005 & 0.004 & 0.006 & 0.005 & 0.004 & 0.006 & 0.007 & 0.004 & 0.007 \\
\multicolumn{1}{c|}{RSF}& 0.011 & 0.002 & 0.012 & 0.004 & 0.004 & 0.006 & 0.004 & 0.005 & 0.009 \\
\multicolumn{1}{c|}{DeepHit}& 0.010 & 0.003 & 0.010 & 0.004 & 0.004 & 0.005 & 0.007 & 0.005 & 0.009 \\
\multicolumn{1}{c|}{Cox-Time}& 0.005 & 0.004 & 0.008 & 0.005 & 0.005 & 0.007 & 0.005 & 0.006 & 0.008 \\
\multicolumn{1}{c|}{DSM}& 0.009 & 0.004 & 0.008 & 0.007 & 0.003 & 0.008 & 0.006 & 0.003 & 0.008 \\
\multicolumn{1}{c|}{DHA}& 0.008 & 0.004 & 0.008 & 0.008 & 0.004 & 0.009 & 0.005 & 0.004 & 0.008 \\ \hline
\multicolumn{1}{c|}{CoxPH}& 0.006 & 0.002 & 0.007 & 0.006 & 0.004 & 0.009 & 0.003 & 0.004 & 0.006 \\
\multicolumn{1}{c|}{TIMENAM}& 0.008 & 0.002 & 0.008 & 0.008 & 0.004 & 0.009 & 0.006 & 0.004 & 0.010 \\
\multicolumn{1}{c|}{TIMENA2M}& 0.011 & 0.003 & 0.011 & 0.005 & 0.006 & 0.006 & 0.005 & 0.008 & 0.008 \\
\multicolumn{1}{c|}{ADHAM ($\mathcal{R}$)}& 0.016 & 0.004 & 0.019 & 0.01 & 0.005 & 0.011 & 0.006 & 0.005 & 0.007 \\
\multicolumn{1}{c|}{ADHAM}& 0.008 & 0.004 & 0.008 & 0.007 & 0.004 & 0.009 & 0.007 & 0.004 & 0.010 \\
\hline                      
\end{tabular}
\end{adjustbox}
\vspace{0.2cm}
\caption{Standard error of the sample mean (SEM) results on the \textbf{CKD} dataset.}
\label{tab:ckderr}
\end{table}
\vspace{-1ex}
\begin{table}[h]
\center
\begin{adjustbox}{width=0.9\textwidth}
\begin{tabular}{cccccccccc}
\hline
& \multicolumn{9}{c}{\textbf{FLCHAIN} Dataset}\\ \hline
\multicolumn{1}{c|}{} & \multicolumn{3}{c|}{25$^\text{th}$ Quantile} & \multicolumn{3}{c|}{50$^\text{th}$ Quantile} & \multicolumn{3}{c}{75$^\text{th}$ Quantile} \\ \cline{2-10} 
\multicolumn{1}{c|}{\multirow{-2}{*}{Models}}& C-Index $\uparrow$ & BS $\downarrow$  & \multicolumn{1}{c|}{AUROC $\uparrow$} & C-Index $\uparrow$ & BS $\downarrow$  & \multicolumn{1}{c|}{AUROC $\uparrow$}  & C-Index $\uparrow$  & BS $\downarrow$    & AUROC $\uparrow$ \\
\hline
\multicolumn{1}{c|}{DeepSurv}& 0.004 & 0.001 & 0.004 & 0.005 & 0.001 & 0.005 & 0.004 & 0.002 & 0.005 \\
\multicolumn{1}{c|}{RSF}& 0.005 & 0.001 & 0.006 & 0.005 & 0.002 & 0.006 & 0.003 & 0.001 & 0.004 \\
\multicolumn{1}{c|}{DeepHit}& 0.004 & 0.002 & 0.004 & 0.007 & 0.002 & 0.007 & 0.005 & 0.002 & 0.006 \\
\multicolumn{1}{c|}{Cox-Time}& 0.004 & 0.004 & 0.005 & 0.006 & 0.012 & 0.007 & 0.004 & 0.024 & 0.005 \\
\multicolumn{1}{c|}{DSM}& 0.008 & 0.001 & 0.009 & 0.010 & 0.002 & 0.011 & 0.005 & 0.001 & 0.006 \\
\multicolumn{1}{c|}{DCM}& 0.003 & 0.002 & 0.003 & 0.009 & 0.004 & 0.010 & 0.007 & 0.003 & 0.008 \\
\multicolumn{1}{c|}{DHA}& 0.003 & 0.002 & 0.004 & 0.007 & 0.002 & 0.008 & 0.004 & 0.001 & 0.005 \\\hline
\multicolumn{1}{c|}{CoxPH}& 0.005 & 0.001 & 0.005 & 0.006 & 0.002 & 0.007 & 0.003 & 0.002 & 0.003 \\
\multicolumn{1}{c|}{TIMENAM}& 0.004 & 0.001 & 0.004 & 0.006 & 0.003 & 0.007 & 0.004 & 0.006 & 0.005 \\
\multicolumn{1}{c|}{TIMENA2M}& 0.006 & 0.001 & 0.007 & 0.007 & 0.006 & 0.008 & 0.003 & 0.015 & 0.004 \\
\multicolumn{1}{c|}{ADHAM ($\mathcal{R}$)}& 0.006 & 0.002 & 0.007 & 0.005 & 0.003 & 0.005 & 0.005 & 0.003 & 0.007 \\
\multicolumn{1}{c|}{ADHAM}& 0.008 & 0.002 & 0.008 & 0.008 & 0.002 & 0.008 & 0.006 & 0.003 & 0.007 \\
\hline                      
\end{tabular}
\end{adjustbox}
\vspace{0.2cm}
\caption{Standard error of the sample mean (SEM) results on the \textbf{FLCHAIN} dataset.}
\label{tab:flcerr}
\end{table}
\end{document}